\documentclass[11pt]{article}
\usepackage[margin=0.9in]{geometry}

\usepackage[utf8]{inputenc} 
\usepackage[T1]{fontenc}    
\usepackage{hyperref}       
\usepackage{url}            
\usepackage{booktabs}       
\usepackage{amsfonts}       
\usepackage{nicefrac}       
\usepackage{microtype}      
\usepackage{xcolor}         
\usepackage{xspace}

\usepackage{amsmath,amssymb}
\usepackage{algorithm,algorithmicx}
\usepackage[noend]{algpseudocode}
\usepackage{mkolar_definitions}
\usepackage{graphicx}
\usepackage{enumitem}
\usepackage{wrapfig}
\usepackage{mathtools}

\usepackage{subcaption}
\usepackage{tikz}
\usetikzlibrary{decorations}

\usepackage[parfill]{parskip}
\usepackage[numbers]{natbib}
\usepackage{authblk}

\newif\ifarxiv
\arxivtrue

\newcommand{\algname}{\ensuremath{\texttt{TASID}}}
\newcommand{\RDP}{\ensuremath{\texttt{RDP}}}

\newcommand{\sinit}{s_{\text{init}}}
\newcommand{\Tsim}{T^{\circ}}

\newcommand{\Treal}{T^{\star}}

\newcommand{\Msim}{M^{\circ}}
\newcommand{\Mreal}{M^{\star}}
\newcommand{\blockM}{\mathbf{M}}
\newcommand{\blockMreal}{\mathbf{M}^{\star}}
\newcommand{\Rsim}{R^{\circ}}
\newcommand{\Rreal}{R^{\star}}
\newcommand{\blockTreal}{\mathbf{T}^{\star}}

\newcommand{\obsmap}{q}
\newcommand{\inverseobsmap}{\phi}

\newcommand{\perturbclass}[2]{\mathcal{C}(#1, #2)}
\newcommand{\perturbprob}{\xi}
\newcommand{\perturbradius}{\eta}

\newcommand{\latentpolicy}{\psi}
\newcommand{\policy}{\pi}
\newcommand{\robustlatentpolicy}{\rho}

\newcommand{\latentpolicyset}{\Psi}

\newcommand{\actionpredictor}{\alpha}

\newcommand{\unf}{{\tt Unf}}

\newcommand{\robustdpV}[1]{\tilde{V}_{#1}}
\newcommand{\robustdpQ}[1]{\tilde{Q}_{#1}}
\newcommand{\robustdppolicy}{\tilde{\latentpolicy}}

\newcommand{\Aset}{A}
\newcommand{\Acompset}{A^{c}}

\newcommand{\Expt}{\mathrm{E}}
\newcommand{\Prob}{\mathrm{P}}

\newcommand{\ndecode}{n_{D}}

\newcommand{\tuple}[1]{\langle #1\rangle}
\newcommand{\set}[1]{\{#1\}}
\newcommand{\given}{\mathbin{\vert}}

\newcommand{\bigBracks}[1]{\bigl[#1\bigr]}
\newcommand{\bigParens}[1]{\bigl(#1\bigr)}
\newcommand{\bigGiven}{\mathbin{\bigm\vert}}

\newcommand{\Alg}{\textup{\texttt{Alg}}} 
\newcommand{\mumin}{\mu_{\textrm{min}}}

\newcommand{\figsqueeze}{}
\newcommand{\mycomment}[1]{\hfill\textcolor{blue}{//#1}}

\date{}
\title{Provably Sample-Efficient RL with Side Information about Latent Dynamics}

\author[1]{Yao Liu \thanks{yao.liu.chn@gmail.com; The major part of this work is done when Yao Liu was an intern at Microsoft Research.}}
\author[2]{Dipendra Misra \thanks{dipendra.misra@microsoft.com}}
\author[2]{Miro Dudík \thanks{mdudik@microsoft.com}}
\author[2]{Robert E. Schapire \thanks{schapire@microsoft.com}}

\affil[1]{ByteDance, Bellevue, WA}
\affil[2]{Microsoft Research, New York, NY}

\begin{document}

\maketitle

\begin{abstract}
    We study reinforcement learning (RL) in settings where observations are high-dimensional, but where an RL agent has access to abstract knowledge about the structure of the state space, as is the case, for example, when a robot is tasked to go to a specific room in a building using observations from its own camera, while having access to the floor plan. We formalize this setting as transfer reinforcement learning from an \emph{abstract simulator}, which we assume is deterministic (such as a simple model of moving around the floor plan), but which is only required to capture the target domain's latent-state dynamics approximately up to \emph{unknown} (bounded) perturbations (to account for environment stochasticity). Crucially, we assume no prior knowledge about the structure of observations in the target domain except that they can be used to identify the latent states (but the decoding map is unknown). Under these assumptions, we present an algorithm, called $\algname$, that learns a robust policy in the target domain, with sample complexity that is polynomial in the horizon, and \emph{independent} of the number of states, which is not possible without access to some prior knowledge.
In synthetic experiments, we verify various properties of our algorithm and show that it empirically outperforms transfer RL algorithms that require access to ``full simulators'' (i.e., those that also simulate observations).\looseness=-1

\end{abstract}

\section{Introduction}
\label{sec:introduction}
When learning from scratch, reinforcement learning (RL) in the real world
can be very expensive.
For example, a robot learning to navigate in a
building might need to explore every possible state or
location, which can be painstakingly costly and time-consuming.
Sometimes, however, it is not necessary to begin such a learning
process from scratch.
For instance, in the robot example, we might have access to a general
floor map of the building.
How can this kind of high-level but imprecise information
be used to learn how to operate in the environment more quickly and
more effectively?

In this paper, we study how to effectively leverage prior information
in the form of such ``abstract'' descriptions of the environment.
We formalize this abstract description as an
``abstract simulator,'' which, like a map, can be used as an imperfect
model.
Importantly, our abstract simulators differ from more standard
simulators in that they only focus on the ``latent structure'' of the
environment dynamics, not on the observations that might be
experienced by an agent in the environment.

In general, fully faithful simulators of even the latent dynamics
might be hard to build, for instance, due to the difficulty of exactly
modeling all probabilistic outcomes, as when the robot's actions do not
have exactly their intended effect.
More complex simulators might be more faithful, but simpler simulators
might be easier to build and also more computationally tractable.

In this paper, we address this trade-off with
a particular design choice:
First, we assume that the abstract simulator is
deterministic.
Indeed, an ordinary map, which implicitly represents
what new position will be reached by a particular action, is such a
deterministic model.
Compared to fully probabilistic models,
deterministic ones are especially simple to build and work with.

\begin{figure*}
    \centering
    \includegraphics[scale=0.21]{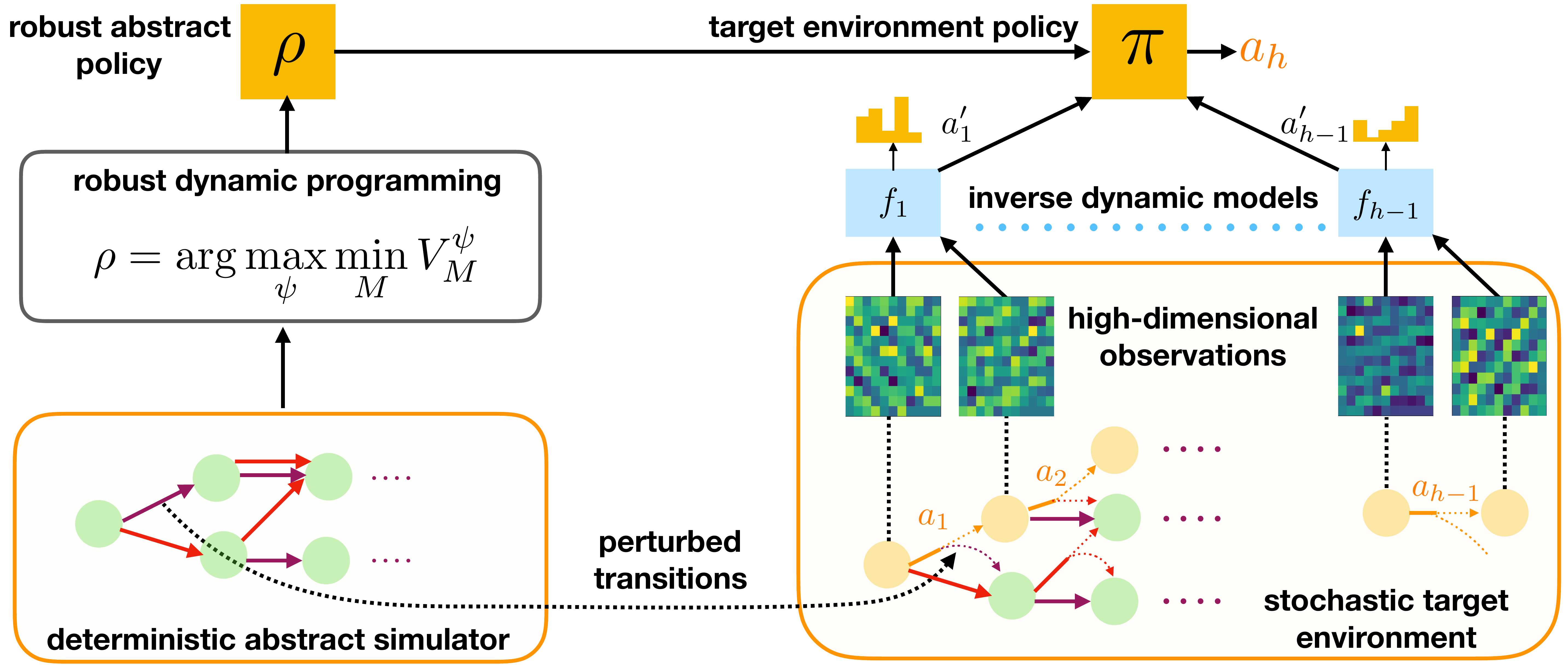}
    \caption[hello]{Overview of our setup. We propose an algorithm $\algname$ that learns a robust policy $\pi$ in the target environment using side information in the form of a fully specified deterministic model, called an ``abstract simulator.'' Abstract simulator approximates (possibly stochastic) latent dynamics of the target environment and thus serves as an idealized description of the the environment.
    The policy $\pi$ selects its actions according to the robust abstract policy $\rho$, but the latent states need to be inferred from high-dimensional observations using learned inverse dynamics models.}
    \label{fig:intro_flowchart}
    \figsqueeze
\end{figure*}

On the other hand, a deterministic model will not in general
capture the ``noise'' of real-world dynamics.
Therefore, we seek algorithms that will learn
policies that are robust to errors or ``perturbations'' in the model
represented by the abstract simulator.
To this end,
in Section~\ref{sec:algorithm},
we present a new algorithm called $\algname$
that provably finds a policy achieving a certain level of performance
for any target environment whose dynamics can be reasonably
approximated by a given abstract simulator,
up to some quantifiable perturbation level, and with arbitrarily different observations.
(\pref{fig:intro_flowchart} visualizes our setup and approach.)

Although our focus here is on transfer from an abstract simulator, and we emphasize the relative ease of constructing abstract simulators, our approach could also be used in a setting similar to theirs. In the source domain, we could use an existing rich-observation approach \cite[e.g.,][]{misra2020kinematic} to infer latent structure, which could then be used as an abstract simulator to speed up learning in the target domain using $\algname$. 

In the robot example, although a map is intuitively helpful for
reducing the need for exploration, there is still much to be inferred,
for instance, how
locations on the map correspond to locations in the building, and,
even more challenging, how they correspond to rich, high-dimensional
observations experienced via cameras or other sensors.
Our algorithm solves these challenges, building on prior work on
learning from such observations~\citep{krishnamurthy2016pac,jiang2017contextual}.

Furthermore,
our theoretical guarantees show that we do indeed save dramatically on
sample complexity (measured by the required number of
interactions with the target environment) by not
having to explore the entire environment, and instead focusing the
task of learning observations just to the states that lead towards
accomplishing the task.
Indeed, our algorithm's sample complexity is entirely
\emph{independent} of the size of the state space, assuring efficiency
even when the state space is extremely large, as is often the case.
We are able to achieve this result because of specific assumptions
and criteria for success, as outlined above, namely,
near-deterministic latent dynamics, knowledge of the abstract simulator, and the benchmark of a robust policy rather than an optimal policy.

In Section~\ref{sec:experiments},
we empirically evaluate $\algname$ in two domains. The first represents a challenging problem requiring strategic exploration. 
We show that $\algname$ is able to efficiently achieve the robust policy value while the PPO algorithm~\citep{schulman2017proximal} augmented with exploration bonus \citep{burda2018exploration} and domain randomization~\citep{tobin2017domain} does not solve the problem. The second domain is a visual grid-world that tests the robustness of $\algname$ to dynamics perturbation and the change of state-space size. We show that $\algname$ succeeds empirically in agreement with our theory.\looseness=-1

\section{Related Work}
\label{sec:related_work}
\textbf{Sim-to-real transfer learning.} The problem of reducing sample complexity of RL in a real-world target domain with the use of a simulator is also known as ``sim-to-real transfer.'' In this setting, a policy is trained on a set of simulators and then fine-tuned and evaluated in real-world environment(s). Most sim-to-real algorithms, such as domain randomization~\citep{sadeghi2016cad2rl,tobin2017domain}, rely on the similarity between the observation spaces in the simulator and real environments~\citep{peng2018sim,andrychowicz2020learning}, or even assume the same observation space, but possibly different action spaces~\citep{christiano2016transfer,song2020provably}.
In contrast, we focus on using an abstract simulator which does not model observations, but assumes the same action space.

\textbf{Transfer RL with different observation spaces.} In transfer RL, most previous work \citep{taylor2007transfer,mann2013directed} focuses on transferring under certain prior knowledge between the two observation spaces. Very recent work from \citet{sun2022transfer} is closest to the settings considered in this paper with a different solution than ours. \citet{sun2022transfer} tackled drastic changes in observation spaces and proposed an algorithm transferring the policy from source domain via learning a sufficient
representation. It provide an asymptotic guarantee of the policy learning given some representation condition, and empirical validation of the algorithm. However, \citet{sun2022transfer} does not gives finite-sample guarantees of the policy learning and any error bounds of learning the representation out of the deterministic transition case. \citet{van2021component} proposed a deep RL algorithm for transfer learning with very different visual observation spaces by learning the abstract states, while this paper focuses more on a provable sample efficiency guarantee.

\textbf{Provably efficient RL with rich observations.} 
While discussing the RL problem with different levels of state abstraction, we adopt a framework used in the line of work about solving Markov Decision Processes (MDPs) with rich observation~\citep{krishnamurthy2016pac,jiang2017contextual}. In these problems, the agent receives high-dimensional observations generated from a much simpler latent state space. Further, the observation space is \emph{rich} in the sense that for every observation, there is a unique latent state that can generate it. Thus this setting is still an MDP and different from partially-observed MDPs (POMDPs) where sample-efficient learning is, in general, intractable. Comparatively, there is limited work on algorithms that can transfer knowledge between two Block MDPs in a provably-efficient manner. In this work, we initiate the theoretical study of transfer learning from one Block MDP to another.

\textbf{Robust reinforcement learning.} Our work builds on top of dynamic programming algorithms for robust reinforcement learning~\citep{iyengar2005robust,bagnell2001solving}. Our algorithm assumes the uncertainty set has a particular structure under which the max-min optimization in the above work has a a closed-form solution. While our solution concept is also to learn a robust policy against a certain uncertainty MDP set, the task that we study is quite different. Our goal is transfer while simultaneously learning to recognize the observations.

\section{Problem Setting}
\label{sec:problem_setting}
We next formalize our modeling setup and assumptions.
As a running example, we consider a navigation problem on a floor of a building, where the goal is to reach a certain location by a robot that can turn left and right and move forward. The robot senses the environment via lidar readings and/or a camera, but it does not have access to its position or orientation.
We use the notation $\Delta(S)$ for the set of probability distributions over a set $S$, and $[n]$ to denote the set $\set{1,2,\dotsc,n}$.

\subsection{Target environment}

We assume that the target environment is a \emph{block MDP} (see, e.g., \citealp{du2019provably,misra2020kinematic,zhang2020invariant}). In a block MDP, observations are high-dimensional (e.g., lidar and camera readings), but emitted by a finite number of latent states (e.g., location and orientation of a robot); latent states are not directly observed, but they can be determined from the observations they emit.\looseness=-1

Formally, a block MDP is a triple \textit{$\blockM=\tuple{M, \Xcal,q}$}. The first component of the triple is a standard episodic MDP $M=\tuple{\Scal,\Acal,\sinit,H, T, R}$, referred to as the \emph{latent MDP}, with a finite state space~$\Scal$, referred to as the \emph{latent state space}, a finite action space $\Acal$, initial latent state $\sinit$, horizon $H$, transition function $T:[H]\times\Scal\times\Acal\to\Delta(\Scal)$, and reward function $R:[H]\times\Scal\times\Acal\to [0,1]$; transition probabilities and reward function values are written as $T_h(s'\given s,a)$ and $R_h(s,a)$. The second component of the block MDP triple is an \emph{observation space} $\Xcal$, which is typically large and possibly infinite. The final component is an \emph{emission function} $q:\Scal\to\Delta(\Xcal)$, which describes the conditional distribution over observations given any latent state; its values are written as $q(s\given x)$.

We assume that $q$ satisfies the block MDP assumption, meaning that there exist disjoint sets $\set{\Xcal_s}_{s\in\Scal}$ such that if $x\sim q(\cdot\given s)$ then $x\in\Xcal_s$ with probability 1. This means that there exists a \emph{perfect decoder} $\phi_\blockM: \Xcal \rightarrow \Scal$ that maps an observation $x$ to the unique state $\phi_\blockM(x)$ that emits it.

An agent interacts with a block MDP in a sequence of episodes, each generated as follows:
Initially, $s_1 = \sinit$.
At each step $h=1,\ldots,H$, the agent observes
$x_h \sim \obsmap(\cdot\given s_h)$, then takes action $a_h$, accrues reward
$r_h = R_h(s_h, a_h)$, after which the MDP transitions to state
$s_{h+1} \sim T_h(\cdot\given s_h, a_h)$. The agent does \emph{not} observe the latent states $s_h$, only the observations $x_h$ and rewards $r_h$. Observations and actions up to step $h$ are denoted $\xb_{1:h}$ and $\ab_{1:h}$.

We denote the block MDP describing the target environment by $\blockMreal =\tuple{\Mreal, \Xcal, q^\star}$ where $\Mreal=\tuple{\Scal,\Acal,\sinit,H, \Treal, \Rreal}$ is the target latent MDP. The perfect decoder for $\blockMreal$ is denoted as~$\phi^\star$.\looseness=-1

\textbf{Practicable policy.}
Behavior of an agent in a target environment is formalized as a (non-Markovian) \emph{practicable policy}. A practicable policy prescribes which action to take given any sequence of observations, i.e., it is a mapping $\pi:\Xcal^{\le H}\to\Acal$, where $\Xcal^{\le H}=\cup_{h=1}^H\Xcal^h$. The action taken by $\pi$ on $\xb_{1:h}$ is written $\pi_h(\xb_{1:h})$.
The expected sum of rewards in $\blockM$, when actions $a_h$ are chosen according to a practicable policy $\pi$, i.e., when $a_h=\pi_h(\xb_{1:h})$, is referred to as the \emph{value of $\pi$ in $\blockM$} and denoted
$V^\pi_\blockM=\Expt_{\blockM,\pi}[r_1+r_2+\dotsb+r_H]$; the subscript in the expectation signifies the probability distribution over episode realizations when the environment follows $\blockM$, and actions are chosen according to $\pi$.

We also allow practicable policies to be (Markovian) mappings
$\pi:[H]\times\Xcal\to\Acal$ which choose every action according to the last observation,
so that $a_h=\pi_h(x_h)$ for all $h$.
Whether a particular practicable policy is Markovian or not will generally be clear from context.

\subsection{Abstract simulator}

Our learning algorithm can interact with the target environment, but it additionally has access to an \emph{abstract simulator}, which provides an idealized and abstracted version of the target environment.
Formally, an abstract simulator is an episodic MDP denoted $\Msim = \tuple{\Scal,\Acal,\sinit,H,\Tsim,\Rsim}$, with the same state space, action space, start state and horizon as the latent MDP $\Mreal$, but not necessarily the same transition and reward functions. Furthermore, we assume that abstract simulator is deterministic:
\begin{assumption}
\label{asm:deterministic}
The abstract simulator $\Msim$ is deterministic, i.e., $\Tsim_h(s_{h+1} \given s_h, a_h)\in\set{0,1}$.\looseness=-1
\end{assumption}

In order for the abstract simulator to be useful, it must approximate target environment. In this paper we assume that the target environment can be viewed as a ``perturbed'' version of the abstract simulator, using a notion of perturbation inspired by the concept of trembling-hand equilibria from extensive-form games \citep{selten1975reexamination}. Specifically, we say that an MDP $M'$ is an $\eta$-perturbation of another MDP $M$ if its dynamics can be realized by following $M$'s dynamics while distorting agent actions according to some (unknown) ``noise'' distribution, referred to as $\xi$ in the definition below, which keeps actions unchanged with probability at least $1-\eta$:

\begin{definition}[$\eta$-perturbation]
\label{def:perturb_mdp_class}
We say that an MDP $M' =\tuple{\Scal,\Acal,\sinit,H, T', R'}$ is an \emph{$\eta$-perturbation} of an MDP $M=\tuple{\Scal,\Acal,\sinit,H, T, R}$
if there exists a function $\xi:[H]\times\Scal\times\Acal\to\Delta(\Acal)$ that satisfies $\xi_h(a\given s,a)\ge 1-\eta$ for all
$h$, $s$, $a$,
and such that
\begin{align*}
&\textstyle
   T'_h(s'\given s,a)
   = \sum_{a' \in \Acal} T_h(s'\given s,a')\xi_h(a'\given s,a)
\\
&\textstyle
   R'_h(s,a)
   = \sum_{a' \in \Acal} R_h(s,a')\xi_h(a'\given s,a)
\end{align*}
for all
$h$, $s$, $a$, $s'$; thus MDP $M'$ can be viewed as following the dynamics of $M$ in which each action $a$ is stochastically replaced (``perturbed'') according to $\xi$.

The set of all $\eta$-perturbations of $M$ is denoted $\perturbclass{M}{\eta}$.
\end{definition}

We assume that target environment is at $\eta$-perturbation of the abstracted simulator for a value of $\eta<0.5$. Thus, most of the time the target environment transitions after each action ``as intended'' (i.e., following known dynamics of the abstract simulator), but with a probability at most $\eta$ it may depart from the intended action due to an inherent, but unknown stochasticity:

\begin{assumption}\label{assum:perturbed}
$\Mreal$ is an $\eta$-perturbation of the abstract simulator $\Msim$ for some $\eta < 0.5$.
\end{assumption}

\textbf{Abstract policy.}
Behavior of an idealized agent that can directly access latent state is formalized as a (Markovian) \emph{abstract policy}.
An abstract policy prescribes what action to take in each state~$s$ at a given step~$h$, i.e., it is a mapping $\latentpolicy: [H] \times \Scal\to\Acal$; we write $\latentpolicy_h(s)$ for the action taken by $\latentpolicy$ in step $h$ and state $s$. The expected sum of rewards in an episodic MDP $M=\tuple{\Scal,\Acal,\sinit,H, T, R}$, when following $\latentpolicy$ is called the \emph{value of $\latentpolicy$ in $M$} and denoted
$V^\latentpolicy_M=\Expt_{M,\latentpolicy}[r_1+r_2+\dotsb+r_H]$.

\textbf{Robust abstract policy.}
Since the latent MDP in the target domain is a perturbation of the abstract simulator, we will seek to obtain policies that are robust to \emph{any} allowed perturbation. For a given abstract simulator $\Msim$ and the perturbation level $\eta$, we define a \emph{robust abstract policy} $\rho$ to be a policy that achieves the largest possible value under the worst-case choice of perturbation:
\begin{align}\label{eqn:robust_policy_main_paper}
    \robustlatentpolicy  =
    \adjustlimits\argmax_{\psi \in \latentpolicyset}\min_{M \in \perturbclass{\Msim}{\perturbradius}} V^{\psi}_{M},
\end{align}
where $\latentpolicyset$ is the set of all mappings from $[H] \times \Scal$ to $\Acal$.

In a natural way, a robust abstract policy $\robustlatentpolicy$ can
be composed with the perfect decoder $\phi^\star$ to obtain a
(Markovian) practicable policy
$\robustlatentpolicy \circ \phi^\star:[H]\times\Xcal\rightarrow\Acal$
mapping observations in the target environment
to actions while still maximizing the worst-case
reward among perturbations of $\Msim$.
We aim for
algorithms that find practicable policies that perform almost as well as this \emph{robust practicable policy}.

\subsection{The learning setting}
\label{sec:formal-learning-setting}

We can now formally define our learning setting.
A learning algorithm $\Alg$ in this setting receives as input a
deterministic abstract simulator (episodic MDP)
$\Msim = \tuple{\Scal,\Acal,\sinit,H,\Tsim,\Rsim}$,
meaning it receives the entire MDP
represented in tabular form (or some other computationally convenient
form).
The algorithm is also provided with \emph{oracle access} to a target
environment (block MDP),
$\blockMreal =\tuple{\Mreal, \Xcal, q^\star}$.
This means that the algorithm cannot directly access $\blockMreal$
itself, but can interact with it as an agent would, executing actions $a_h$,
and receiving back observations $x_h$ and rewards $r_h$ in a sequence of episodes,
as described above.
Finally, $\Alg$ is given parameters
$\epsilon>0$, $\delta>0$, and $\eta<0.5$.
It is assumed that $\Mreal$ is an $\eta$-perturbation of $\Msim$.

After interacting with $\blockMreal$,
the algorithm outputs a practicable policy $\pi$.
The goal of learning is for $\pi$ to have value
almost as good as the robust practicable policy with high probability,
that is, for
$\smash{V^{\pi}_{\blockMreal} \ge
V^{\robustlatentpolicy \circ \phi^\star}_{\blockMreal} - \epsilon}$
with probability at least $1-\delta$
(where probability is over the algorithm's randomization as well as
randomness in the interactions with the target environment).
Furthermore, we require the number of episodes executed by the
algorithm before outputting a policy $\pi$ to be bounded by a
polynomial in the number of actions $|\Acal|$, the horizon $H$,
$1/\epsilon$, $1/\delta$, and $1/(1-2\eta)$.
Note importantly that this polynomial must have no explicit dependence on the
number of states $|\Scal|$ or observations $|\Xcal|$ in the target environment. An algorithm that satisfies these criteria (given the stated
assumptions) is said to achieve an
\emph{efficient transfer from abstract simulator}.

\section{Main Algorithm}
\label{sec:algorithm}

Our main contribution is an algorithm $\algname$, which achieves efficient transfer from an abstract simulator. The algorithm operates in two stages. In the first stage, it determines the robust abstract policy $\rho$ for the provided abstract simulator via robust dynamic programming (\pref{alg:robust_dp}). In the second stage, it interacts with the target environment (via oracle access) in order to learn to predict the current latent state based on the current history. The learnt decoding map is then composed with the robust abstract policy to obtain the practicable policy that is returned by the algorithm (see \pref{alg:trembling_hand_decoding}).

\subsection{Robust dynamic programming for abstract simulator}
\label{sec:robust_algorithm}

We obtain the robust abstract policy by instatiating the robust dynamic programming algorithm of \citet{bagnell2001solving} and \citet{iyengar2005robust} to our specific notion of perturbation. The algorithm proceeds by filling out values of the \emph{robust value function} $\smash{\tilde{V}}$, which quantifies the largest sum of rewards achievable starting at any step $h$ and state $s$, when assuming the worst-case perturbation of the input MDP $\Msim$. Specifically,
\[
  \robustdpV{h}(s) =
  \adjustlimits
  \max_{\psi \in \latentpolicyset}
  \min_{M \in \perturbclass{\Msim}{\perturbradius}}
  \Expt_{M,\latentpolicy}\bigBracks{r_h+r_{h+1}+\dotsb+r_H\bigGiven s_h=s}.
\]
Similar to standard dynamic programming, the value function in robust dynamic programming can be filled out beginning with $h=H+1$, where we have $\smash{\robustdpV{h}(s)=0}$, and proceeding backward. In our case, the values $\smash{\robustdpV{h}(s)}$ can be obtained from $\smash{\robustdpV{h+1}(s)}$ using a closed-form expression (\pref{line:robust-value}), which leverages intermediate values $\smash{\robustdpQ{h}(s,a)}$. Note that $\tilde{Q}$ is not quite the robust state-action value function, because it assumes that the action at step $h$ is left unperturbed (and only considers the worst-case perturbation in the following steps). The function $\smash{\tilde{Q}}$ is used to derive the robust policy (\pref{line:robust-extract-policy}).

\begin{theorem}
\label{thm:robust_dp_value}
For any abstract simulator $\Msim$ and any perturbation level $\eta$, \pref{alg:robust_dp} returns the robust policy $\strut\smash{\robustlatentpolicy = \argmax_{\latentpolicy \in \latentpolicyset} \min_{M \in \perturbclass{\Msim}{\perturbradius}} V^{\latentpolicy}_{M}}$.
\end{theorem}
(The proof of this theorem and all other proofs in this paper are deferred to the appendix.)

\begin{algorithm}[t]
\caption{Robust Dynamic Programming. $\RDP(\Msim,\eta)$}
\label{alg:robust_dp}
\begin{algorithmic}[1]
\Statex\textbf{Input:} An episodic MDP $\Msim=\tuple{\Scal,\Acal,\sinit,H,\Tsim,\Rsim}$, perturbation level $\eta$.
\smallskip
\State $\smash{\robustdpV{H+1}(s)} \leftarrow 0$ for all $s \in \Scal$
\For{$h=H, \dots, 1$}
    \State \textbf{for} all $s\in\Scal,a\in\Acal$:
           $\quad\robustdpQ{h}(s,a)
               \leftarrow \Rsim_h(s,a) + \sum_{s'} \Tsim_h(s'\given s,a) \robustdpV{h+1}(s')$
           \label{line:bellman-eqn}
    \State \textbf{for} all $s\in\Scal$:$\hphantom{,a\in\Acal}$
           $\quad\robustdpV{h}(s)
               \leftarrow  (1- \perturbradius) \max_a \robustdpQ{h}(s,a) + \perturbradius \min_a \robustdpQ{h}(s,a)$
       \label{line:robust-value}
\EndFor
\State \textbf{Return} $\rho$ defined by
    $\rho_h(s) = \argmax_{a} \robustdpQ{h}(s,a) $
    \label{line:robust-extract-policy}
\end{algorithmic}
\end{algorithm}

\subsection{Learning a decoder in the target environment}
\label{sec:decoding_algorithm}

\newcommand{\xind}{\hspace*{\algorithmicindent}}
\begin{algorithm}[t]
\caption{Transfer from Abstract Simulator using Inverse Dynamics. $\algname(\blockMreal,\Msim,\Fcal,\eta,\epsilon,\delta)$ }
\label{alg:trembling_hand_decoding}
\begin{algorithmic}[1]
\Statex \textbf{Input:}
        Oracle access to target environment $\blockMreal$,
        deterministic abstract simulator $\Msim$,
\Statex \hphantom{\textbf{Input:}}
        optimization-oracle access to a function class
        $\Fcal\subseteq\{\Xcal^2 \rightarrow \Delta(\Acal)\}$,
\Statex \hphantom{\textbf{Input:}}
        perturbation level $\eta<0.5$,
        target accuracy $\epsilon>0$, failure probability $\delta$.
\medskip
\State Let $\rho$ be the robust policy returned by $\RDP(\Msim,\eta)$
       \label{line:rho}
\smallskip
\State Let $n_D\defeq \frac{8H^2 |\Acal|^3\ln(|\Fcal|/\delta)}{\epsilon (1-2\perturbradius) ^2}$
\smallskip
\State Let $\policy_{1}(x)\defeq
       \rho_{1}(\sinit)$ for all $x\in\Xcal$
       \mycomment{Define practicable policy in step $h=1$}
       \label{line:pi-1}
\smallskip
\For{$h=1, \dotsc, H-1$}
\State $\Dcal_h \leftarrow \emptyset$
       \label{line:gather}
       \mycomment{Gather dataset $\Dcal_h$ for learning ``inverse dynamics'' in step $h$}
\For{$\ndecode$ times}\label{line:potas-dataset-start}
\State Follow $\policy_{1:h-1}$ for $h-1$ steps to observe $x_{h}$
\State Take action $a_{h}$ uniformly at random and observe $x_{h+1}$
\State $\Dcal_h \leftarrow \Dcal_h \cup \{(x_{h}, a_{h}, x_{h+1})\}$
\EndFor\label{line:potas-dataset-end}
\State $f_h \defeq \arg\max_{f \in \Fcal} \sum_{(x_{h}, a_{h}, x_{h+1})\in \Dcal_h} \ln f(a_h \given x_h, x_{h+1})$
       \label{line:alg_classification}
       \mycomment{Learn ``inverse dynamics''}
\State Define $\alpha_h:\Xcal^2\to\Acal$ as
       \label{line:potas-alpha}
       \mycomment{Define ``shadow action'' decoder}
\Statex\xind\xind
       $\actionpredictor_{h}(x_h, x_{h+1}) = \argmax_{a \in \Acal} f_h(\cdot \given x_h, x_{h+1})$
\smallskip
\State Define $\phi_{h+1}:\Xcal^{h+1}\to\Scal$ such that
       \label{line:potas-phi}
       \mycomment{Define state decoder}
\Statex\xind\xind
       $\phi_{h+1}(\xb_{1:h+1})$ is the state $s_{h+1}$ reached in $\Msim$,
       when starting in $\sinit$ and executing
\Statex\xind\xind\xind
       $a'_1 =  \actionpredictor_1(x_1,x_2), \dotsc, a'_h = \actionpredictor_h(x_{h},x_{h+1})$
\smallskip
\State Define $\policy_{h+1}:\Xcal^{h+1}\to\Acal$ as
       \label{line:potas-policy}
       \mycomment{Define practicable policy}
\Statex\xind\xind
       $\policy_{h+1}(\xb_{1:h+1})=
       \rho_{h+1}\bigl(\phi_{h+1}(\xb_{1:h+1})\bigr)$
\EndFor
\State \Return $\policy=(\policy_{1}, \dotsc, \policy_{H})$
\end{algorithmic}
\end{algorithm}

We cannot directly apply the robust abstract policy in the target environment, because the latent states are not observable. Therefore, we construct a ``state decoder,'' which uses the history of observation in an episode to predict the current latent state; the state decoder is combined with the robust abstract policy to obtain the practicable policy.
Formally, a (non-Markovian) state decoder is a mapping $\phi:\Xcal^{\le H}\to\Scal$. In step $h$ of an episode, the history of previous observations $\xb_{1:h}$ is used as an input, and the decoder predicts the latent state $s_h$. To construct such a state decoder, we crucially leverage the abstract simulator $\Msim$.

The abstract simulator $\Msim$ is deterministic, so a specific sequence of actions $a_1,\dotsc,a_h$ always leads to the same state. The latent MDP in the target environment is a perturbation of the abstract simulator. This means that when an agent takes action $a_h$ in step $h$, the environment most of the time transitions according to $\Tsim(\cdot\given s_h, a_h)$, but sometimes (with probability at most $\eta$) it transitions according to some other action. The action $a_h$ gets replaced with some ``shadow'' action $a'_h$ according to an unknown noise distribution $\xi$, and the latent state then transitions according to $\Tsim(\cdot\given s_h, a'_h)$. If we knew shadow actions $a'_1,\dotsc,a'_h$, we could then recover the current latent state by simulating that same sequence of actions in the abstract simulator.

To obtain shadow actions, we learn an ``inverse dynamics'' model, which predicts $\smash{a'_h}$ from the observations $x_h$ and $x_{h+1}$ (an approach also used in previous work on block MDPs). We learn a separate inverse dynamics model for each step of an episode. In step $h$, we sample triplets of the form $(x_h,a_h,x'_h)$ across multiple episodes in the target environment, and then fit a model $f_h$ for the conditional probability of $a_h$ given $x_h$ and $x'_h$. The model $f_h(a_h\given x_h,x'_h)$ is referred to as an inverse dynamics model, because it ``inverts'' the dynamics represented by the transition function. We show that if we were able to obtain an exact model $f^\star_h$ of the conditional probability, then the action $a$ with the largest probability $f^\star_h(a\given x_h,x'_h)$ would be the correct shadow action.

As is standard in the block MDP literature, in order to fit an inverse dynamics model, we assume access to an optimization algorithm capable of fitting functions from some class $\Fcal$ to data; we call this algorithm an \emph{optimization oracle for $\Fcal$}. The class $\Fcal$ should be sufficiently expressive to approximate the required conditional probability distribution.

We now have all the pieces required to describe our algorithm. The algorithm first constructs the robust abstract policy (\pref{line:rho}), and defines the practicable policy $\pi_1$ at the initial step $h=1$, where the latent state is known to be $\sinit$ (\pref{line:pi-1}). The algorithm then proceeds iteratively to fill in $\pi_2,\dotsc,\pi_H$. In iteration $h$, the algorithm first learns the inverse dynamics model $f_h$ with the help of the optimization oracle (\pref{line:gather}--\ref{line:alg_classification}). The inverse dynamics model is then used to obtain the shadow action decoder $\alpha_h$ (\pref{line:potas-alpha}), which predicts which action caused the transition from $x_h$ to $x_{h+1}$. Using the shadow action decoders up to step $h$, we can construct a state decoder $\phi_{h+1}$ (\pref{line:potas-phi}), which for a given history of observations $\xb_{1:h+1}$, first predicts their corresponding shadow actions $a'_1,\dotsc,a'_h$ and then uses the abstract simulator to determine the state $s_{h+1}$ that they lead to. Finally, using the state decoder $\phi_{h+1}$, we define the practicable policy at the step $h+1$ to return the same action as the abstract robust policy would return on the decoded state (\pref{line:potas-policy}).

\subsection{Sample Complexity of $\algname$}

We next provide the sample complexity analysis of $\algname$, showing that it indeed achieves efficient transfer from an abstract simulator. In addition to Assumptions~\ref{asm:deterministic} and~\ref{assum:perturbed}, we also need to ensure that the function class $\Fcal$ is expressive enough to contain the conditional probability distribution being fitted by the inverse dynamics model. It turns out that this target probability distribution can be expressed in terms of the transition function of the block MDP $\blockMreal$, which is the function $\blockTreal:[H]\times\Xcal\times\Acal\to\Delta(\Xcal)$ equal to
\[
  \blockTreal_h(x'\given x,a)\defeq q^\star(x'\given s'=\phi^\star(x'))T_h\bigParens{s'=\phi^\star(x')\bigGiven s=\phi^\star(x),a}.
\]
Using $\blockTreal$, we can state our assumption on the class $\Fcal$:
\begin{assumption}\label{asm:realizable}
(Realizability with respect to $\blockMreal$) For any $h \in [H]$, there exists $f^\star_h \in \Fcal$ such that $f^\star_h(a \given x, x') = \frac{\blockTreal_h (x'\given x,a )}{\sum_{a'\in \Acal} \blockTreal_h( x'\given x,a' )}$ for all $x$, $x'$ that can occur at steps $h$ and $h+1$.
\end{assumption}

Realizability assumptions are standard in block MDP literature~\cite{du2019provably,agarwal2020flambe}; in practice they are assured by choosing expressive models such as deep neural networks.

We are now ready to state our main theoretical result---sample complexity of $\algname$.
\begin{theorem}
\label{thm:robust_tranfer}
Let $\Msim$ be a deterministic abstract simulator. Let $\blockMreal$ be a
target environment for which $\Mreal$ is an $\eta$-perturbation of $\Msim$
for some $\eta<0.5$. Let $\Fcal$ be a class of functions satisfies Assumption \ref{asm:realizable} with respect to $\blockMreal$.
Then for any $\epsilon>0$ and $\delta \in (0, 1)$, \pref{alg:trembling_hand_decoding} with oracle access to
$\blockMreal$, optimization-oracle access to $\Fcal$, and inputs $\Msim$, $\eta$, $\epsilon$, and $\delta$,
executes $n=\Ocal\bigl(\frac{H^3 |\Acal|^3\ln(|\Fcal|/\delta)}{\epsilon(1-2\perturbradius) ^2}\bigr)$ episodes and returns a practicable policy $\policy$ that with probability at least $1-\delta$ satisfies
\[
V^{\pi}_{\blockMreal} \ge
V^{\robustlatentpolicy \circ \phi^\star}_{\blockMreal} - \epsilon.
\]
\end{theorem}

This result shows that~\pref{alg:trembling_hand_decoding} achieves efficient transfer from abstract simulator, meaning that its sample complexity is independent of the sizes of $|\Scal|$ and $|\Xcal|$. We also achieve a fast rate, $\Ocal(1/\epsilon)$, with respect to the sub-optimality $\epsilon$; the dependence on $\ln|\Fcal|$ is standard. There may be room for improvement in terms of the horizon $H$ and action space size $|\Acal|$, but these were not our focus here.

\subsection{Extensions}

  In the appendix, we show how $\algname$, under additional assumptions, can be extended to settings
  where the abstract simulator has a stochastic start state (but deterministic transitions).

  Although our focus here is on transfer from an abstract simulator, and we emphasize the relative ease of constructing abstract simulators, our approach could also be used for transfer learning between two block MDP environments that share latent space structure. In the source domain, we could use an existing block MDP approach \cite[e.g.,][]{misra2020kinematic} to infer latent structure, which could then be used as an abstract simulator to speed up learning in the target domain using $\algname$.

\section{Experiments}
\label{sec:experiments}

\begin{figure*}[!t]
    \centering
    \begin{minipage}{.35\textwidth}
        \centering
        \includegraphics[width=1\textwidth]{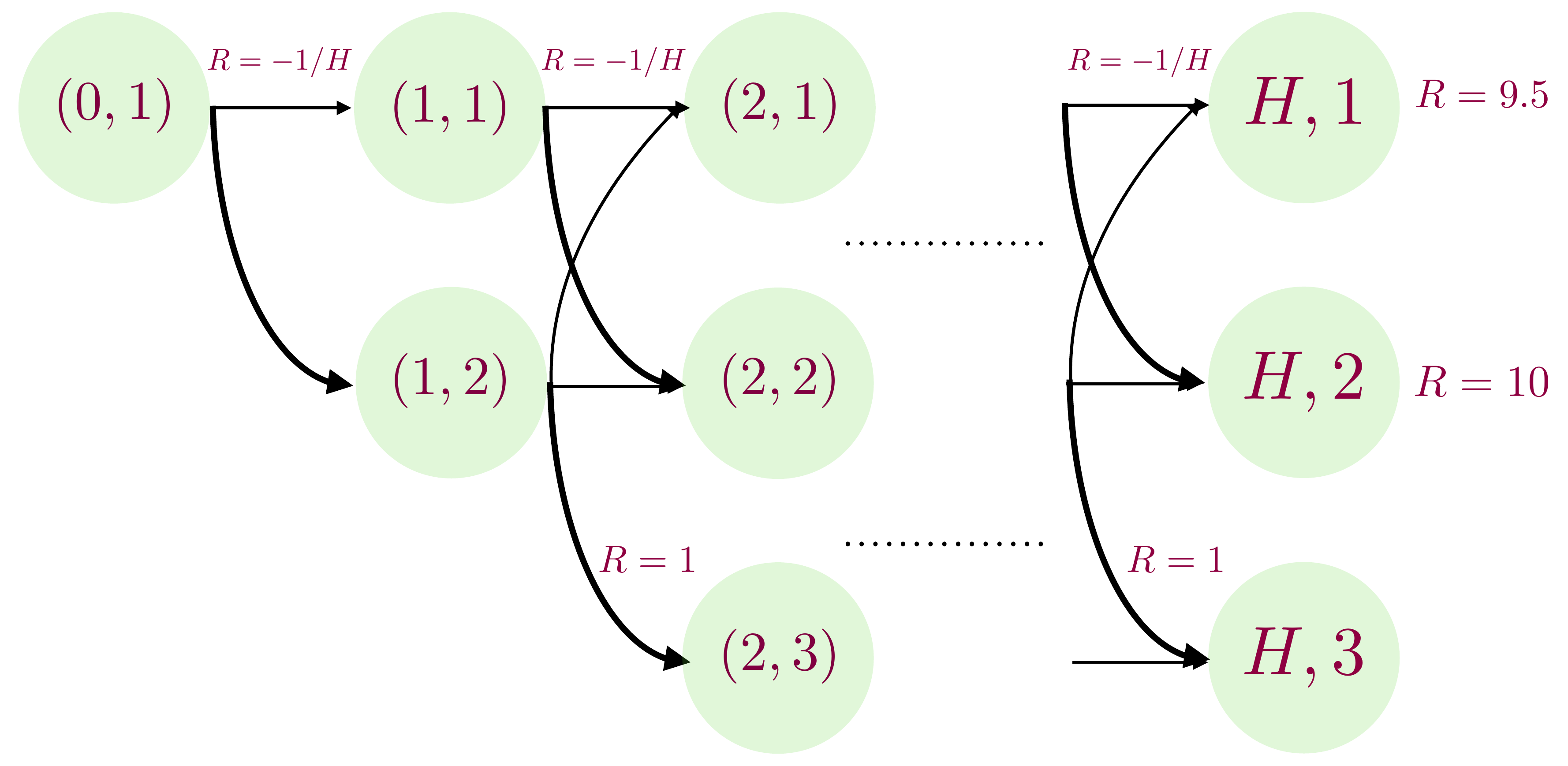}
        \vspace{0.2cm}
        \subcaption{Combination lock environment}
        \label{fig:combolock_mdp}
    \end{minipage}
    \hspace{0.05cm}
    \begin{minipage}{.3\textwidth}
        \centering
       \includegraphics[width=1\textwidth]{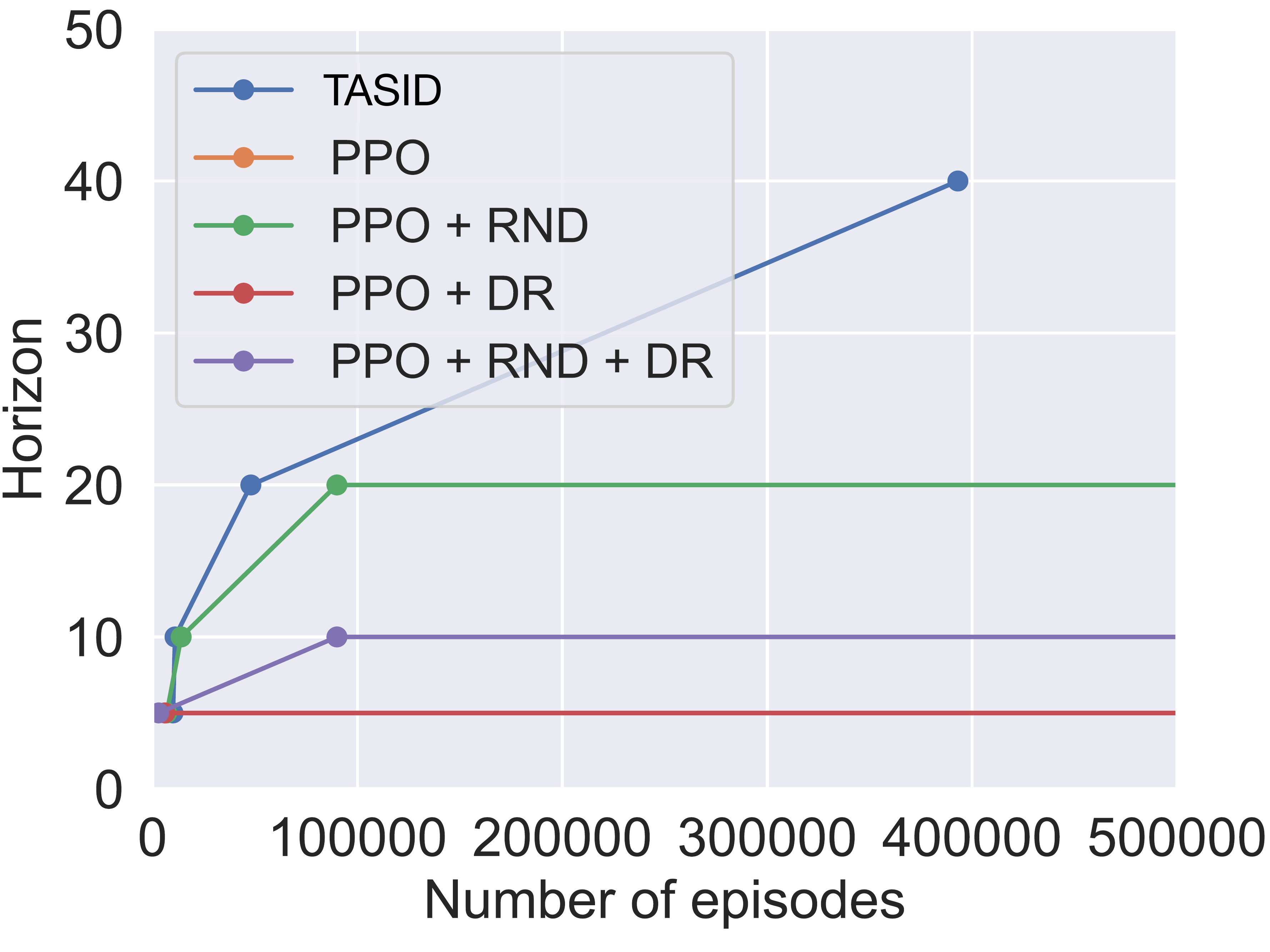}
        \subcaption{Horizon solved \emph{vs} \# episodes}
        \label{fig:combolock_diff_h}
    \end{minipage}
    \hspace{0.05cm}
    \begin{minipage}{.3\textwidth}
        \centering
        \includegraphics[width=1\textwidth]{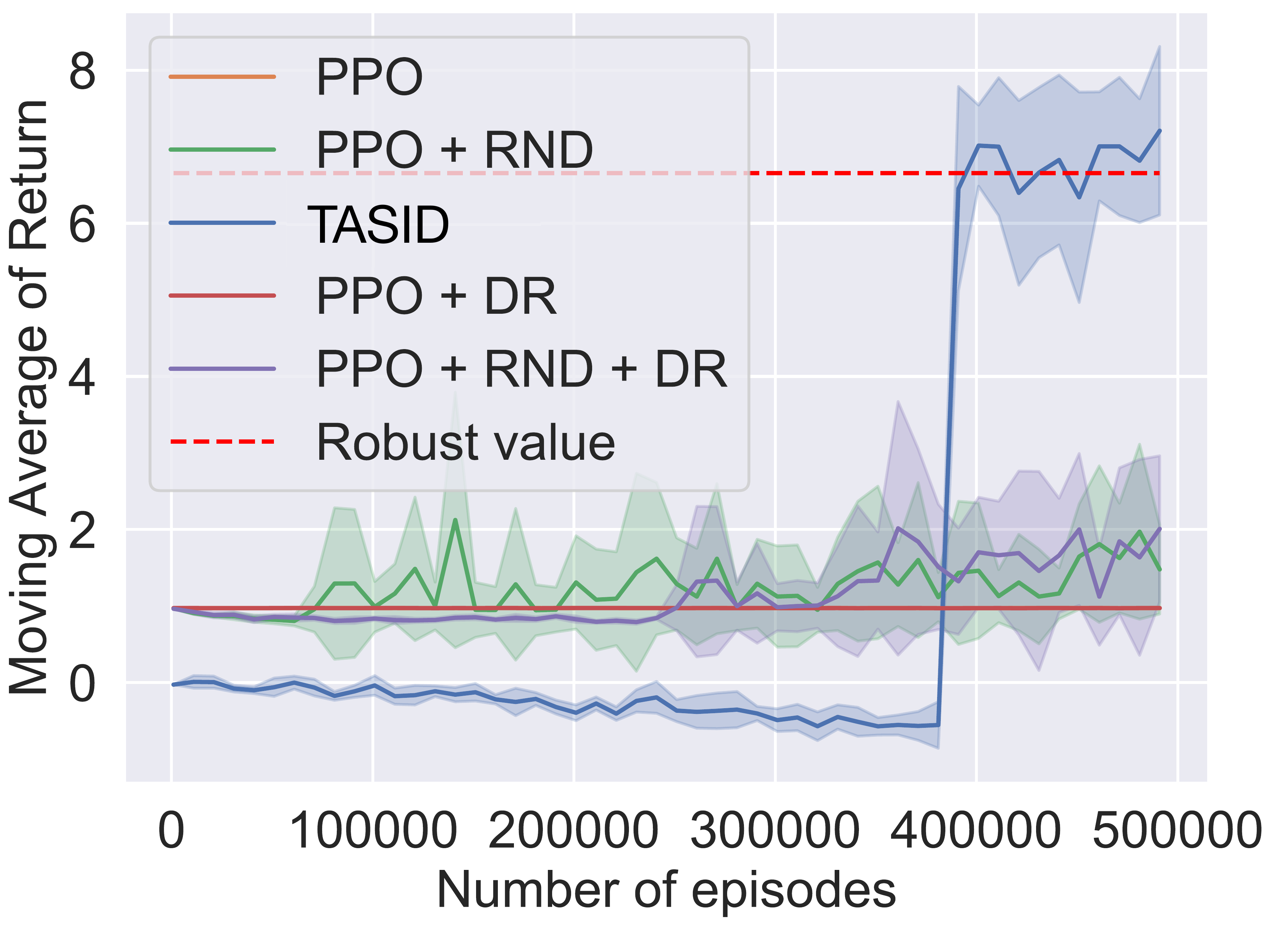}
        \subcaption{Reward curve for $H=40$}
    \label{fig:reward_curve_h40_combolock}
    \end{minipage}
    \caption{
    \textbf{Left:} Latent MDP in the combination lock. \textbf{Center:}
    The largest problem (expressed as horizon $H$) that each algorithm can solve (defined as reaching 95\% of the value of the robust policy) within a given number of episodes; we considered $H=5,10,20,40$ and $|\Acal|=10$. We allow a maximum of $5 \times 10^5$ number of episodes in the target domain. We report median number of episodes across five trials with different seeds (see the appendix for full details).
    \textbf{Right:} Total reward per episode for the problem size $H=40$. $\algname$ is a batch algorithm trained with $400000$ episodes, thus its performance before the end of training is very low and after the end of training stays constant.}
    \label{fig:combo_lock_result}
    \figsqueeze
\end{figure*}

We evaluate $\algname$ on two simulation environments to test four aspects: state-space-size independent sample complexity; scalability to complex visual observations; robustness of a learnt policy; and scalability to large state spaces. We summarize the results here and defer the details, including hyperparameter selection, detailed domain description, algorithm and baselines implementation and a full description of results under various setups to the appendix.

\noindent\textbf{Can $\algname$ solve problems that require strategic exploration?} We theoretically showed that $\algname$ can solve problems that require performing strategic exploration using a small number of episodes. We test this empirically on a challenging environment called \emph{combination lock}.

\noindent\textbf{Combination lock.} We first describe the abstract simulator $\Msim$ visualized in Figure~\ref{fig:combolock_mdp}. It has an action space $\Acal$, horizon $H$, and a state space $\{(h, i): h \in \{0\} \cup [H], i \in [3]\}$, with the initial state $(0, 1)$. As we will see, states $\{(h, 1), (h, 2) \mid h \in [H]\}$ are \emph{good} states from which optimal return is possible, while states $\{(h, 3) \mid h \in [H]\}$ are \emph{bad} states. In state $(h,1)$, one good action leads to state $(h+1, 1)$, and all other $|\Acal|-1$ actions lead to state $(h+1, 2)$. In state $(h, 2)$, one good action leads to state $(h+1, 2)$, another good action leads to state $(h+1, 1)$, and all other $|\Acal|-2$ actions lead to $(h+1, 3)$. All actions in state $(h, 3)$ lead to state $(h+1, 3)$. The identities of good actions are unknown and are different for different states. Any transition to state $(H, 1)$ gives a reward of~$9.5$, whereas transition to state $(H, 2)$ gives a reward of $10$, and transition to state $(H, 3)$ gives a reward of $0$. To ``mislead'' the agent, transitions from $(h, 2)$ to $(h+1,3)$ give a reward of~$1$ and transitions from $(h, 1)$ to $(h+1, 1)$ give a reward of~$-1/H$. All other transitions give a reward of 0.

The latent MDP $\Mreal$ of the target environment is constructed by taking a random $\eta$-perturbation of~$\Msim$ for $\eta = 0.1$. Exploration in both $\Mreal$ and $\Msim$ can be difficult due to misleading rewards and challenging dynamics, where most actions lead to bad states. The optimal policy in $\Msim$ finishes in the state $(H, 2)$ and achieves the value of 10. However, the robust policy will attempt to visit the state $(H, 1)$ which lies on a more stable trajectory. Meanwhile, the optimal policy in $\Msim$ will fail in $\Mreal$, because perturbations are likely to move the agent into a bad state.

Observations in the target environment are real-valued vectors of dimension $2^{\lceil \log_2(H+4) \rceil }$. For a latent state $(h, i)$, the observation is generated by first creating an $(H+4)$-dimensional vector by concatenating one-hot encodings of $h$ and $i$, concatenating it with another $(H+4)$-dimensional i.i.d.\ Bernoulli noise vector, applying a fixed coordinate permutation, element-wise adding Gaussian noise, and finally multiplying the vector with a Hadamard matrix.
The goal of this process is to ``blend'' the informative one-hot encoding with uninformative Bernoulli noise.

We compare $\algname$ with four baselines.
The first is PPO~\citep{schulman2017proximal}, a policy-gradient-based algorithm. The second is PPO-RND, which augments PPO with an exploration bonus based on prediction error~\citep{burda2018exploration}. These two baselines do not use abstract simulator and run on the target environment from scratch.
The other two baselines, PPO+DR and PPO+RND+DR, enhance PPO and PPO+RND with domain randomization (DR), where the policy is pre-trained on a set of randomized block MDPs~\citep{tobin2017domain}. PPO+DR and PPO+RND+DR are designed for transfer RL, but unlike our work, they rely on observation-based similarity. During the pre-training phase of domain randomization, we generate block MDPs by following a similar process as we used to generate the target environment, but when generating the emission function $q$, only the permutation matrix is regenerated (and other components are kept the same as in the target environment).

\begin{figure*}[!t]
    \centering
    \begin{minipage}{.15\textwidth}
        \centering
        \includegraphics[width=\textwidth]{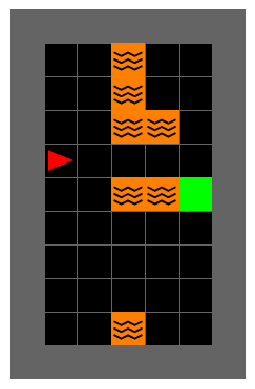}
        \subcaption{Minigrid}
        \label{fig:minigrid_original}
    \end{minipage}
    \begin{minipage}{.15\textwidth}
        \centering
        \includegraphics[width=\textwidth]{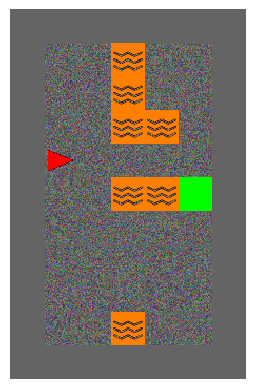}
        \subcaption{Noise}
        \label{fig:minigrid_noisy}
    \end{minipage}
    \begin{minipage}{.33\textwidth}
        \centering
        \includegraphics[width=1\textwidth]{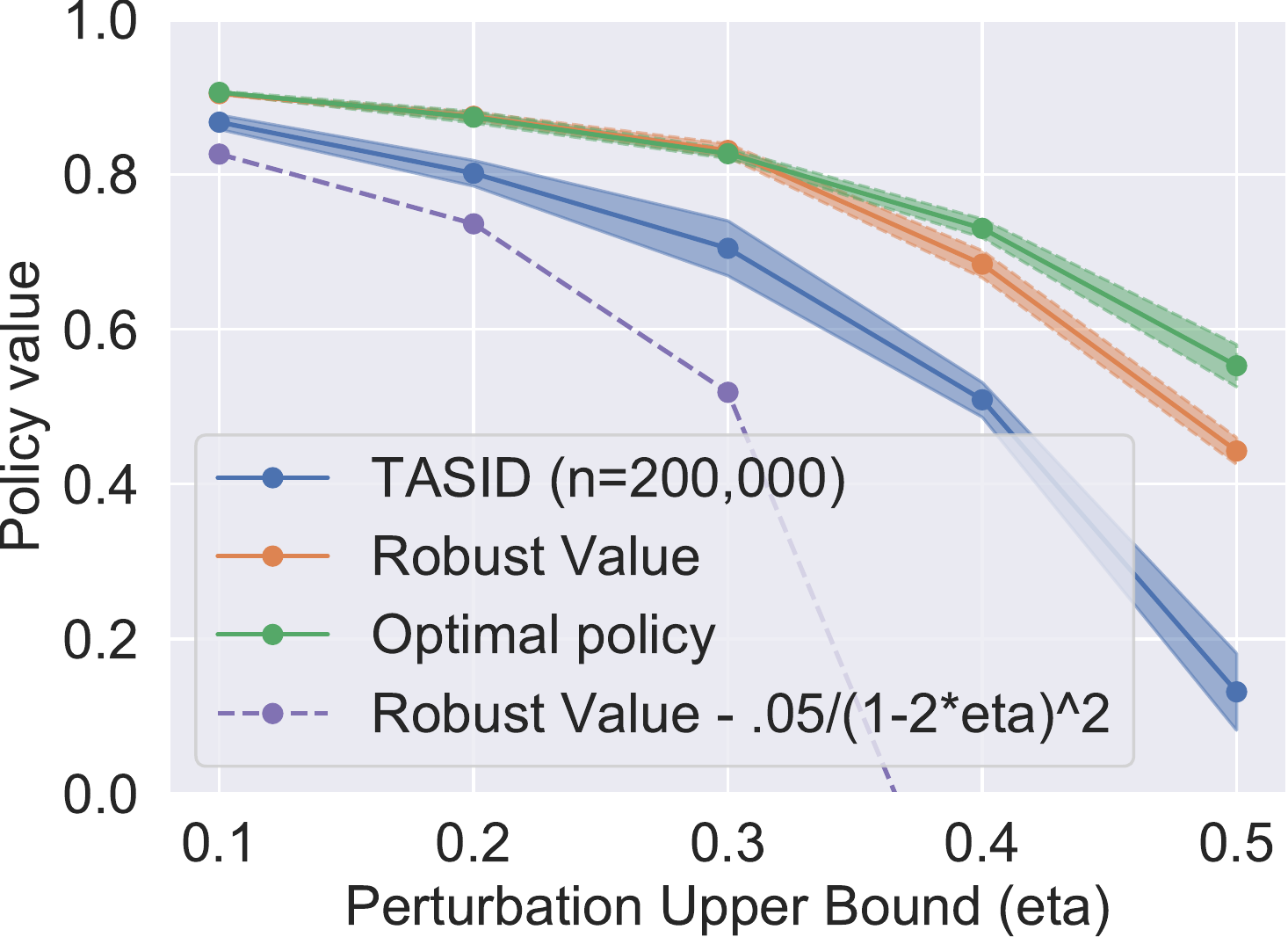}
        \subcaption{Performance \emph{vs} perturbation}
        \label{fig:minigrid_result_alpha}
    \end{minipage}
    \begin{minipage}{.33\textwidth}
        \centering
        \includegraphics[width=1\textwidth]{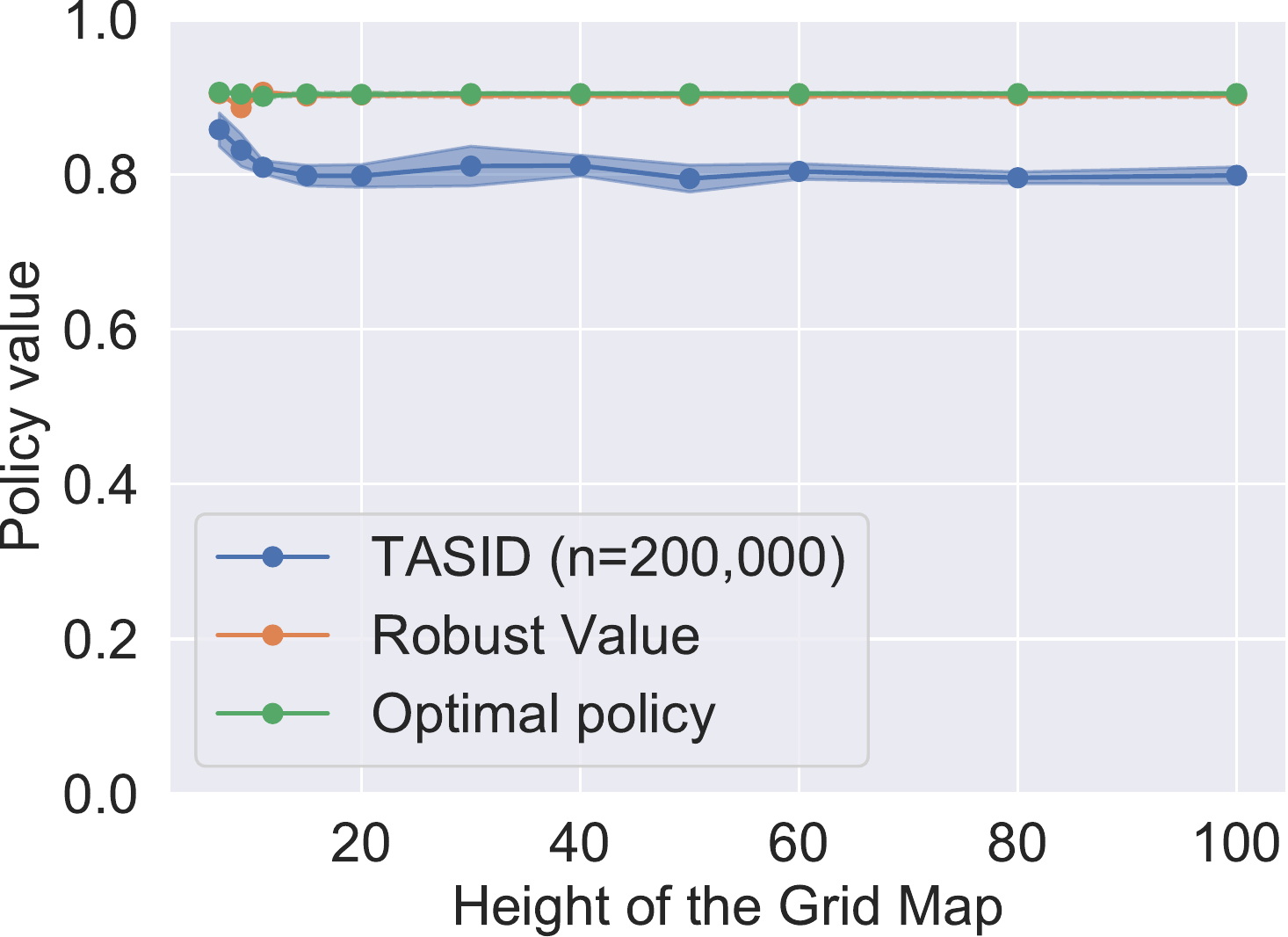}
        \subcaption{Performance \emph{vs} state space size}
        \label{fig:minigrid_result_height}
    \end{minipage}
    \caption{Map of a $7\times11$ MiniGrid environment (a) with noisy observations (b) and five actions: forward, turn left, turn right, turn left + forward, turn right + forward. (c): The performance of \algname\ with different amount of perturbation $\perturbradius$. (d): The performance of \algname\ with different height of the map.\looseness=-1}
    \label{fig:minigrid}
    \figsqueeze
\end{figure*}

\pref{fig:combolock_diff_h} plots the size of the problem (expressed as horizon $H$) that each algorithm can solve (defined as reaching at least 95\% of the value of the robust policy) within a given number of episodes; we considered $H=5,10,20,40$. The plot shows that PPO and PPO+DR fail to solve the problem beyond the smallest size ($H=5$), which is not surprising as they are not designed to perform strategic exploration. PPO+RND can solve the problem until $H=20$, but fails for $H=40$, showing that RND bonus helps to an extent but cannot solve harder problems. Interestingly, PPO+RND+DR underperforms compared to PPO+RND even though it has access to pre-training. We believe this is because RND reward bonus does not work well with domain randomization which can mislead the RND bonus by randomizing the observation. Finally, $\algname$ can solve the problem for all values of $H$ and is more sample efficient than the baselines.\looseness=-1

\noindent\textbf{Can $\algname$ scale to complex observation spaces?} We evaluate $\algname$ in the visual MiniGrid environment~\citep{gym_minigrid} with noisy observations. We test on the grid-world map shown  in~\pref{fig:minigrid_original}.
The green square represents the goal, and orange squares are lava that should be avoided. The agent's position and direction is shown by the red triangle. The agent gets a reward of $1$ on reaching the goal, $-1$ for reaching the lava, and $-0.01$ for every other step. The state encodes the position and direction of the agent and the current time step. There are 5 actions: move forward, turn left, turn right, and combinations of turning and moving. The abstract simulator $\Msim$ models the deterministic transitions on the grid. The target environment is an $\eta$-perturbation of $\Msim$ $\perturbradius = 0.1$. The agent cannot observe the whole map as in Figure \ref{fig:minigrid_original}, but only an image of the $7\times7$ grids ($56\times56$ pixels) in the direction it is facing, with i.i.d.\ random noise added to all black pixels as shown in~\pref{fig:minigrid_noisy}. This type of random noise has been studied in the literature and has been shown to pose challenge to RL algorithms~\citep{burda2018large}.

In this grid-world, the agent can either use the top route to reach the goal, or a bottom route. The top route is shorter but the agent is more likely to visit the lava due to perturbations. Therefore, the robust policy prefers the bottom route which is longer but more robust.

\noindent\textbf{How robust is $\algname$ to different $\perturbradius$?} 
Our theoretical analysis suggests that the suboptimality scales with $\Ocal(\frac{1}{1-2\perturbradius})$. We evaluate this error empirically on gridworld for various values of $\perturbradius$ in $[0.1, 0.5]$ and show the results in~\pref{fig:minigrid_result_alpha}.
The performance matches the theoretical prediction of $v^\robustlatentpolicy -  \Ocal(\frac{1}{1-2\perturbradius})$ until $\perturbradius=0.4$. Interestingly, even though the theoretical guarantee is vacuous at $\perturbradius=0.5$, the algorithm still learns a non-trivial policy ($V^\pi$ greater than the value of random walk policy). 

\noindent\textbf{How does $\algname$'s performance scale with $|\Scal|$?} We showed theoretically that $\algname$'s sample complexity to learn a near-optimal robust policy is independent of the size of the state space. We empirically test this by varying the height of the grid-world environment while keeping the width and the horizon $H$, and overall layout the same. Results are presented in~\pref{fig:minigrid_result_height}. As height changes from $10$ to $100$, the size of the state space increases ten times; however, we do not see a significant drop in the performance of $\algname$ when using a fixed number of episodes in the target environment.

\section{Conclusion and Future Work}
\label{sec:conclusion}
We presented a new algorithm $\algname$ for transfer RL in block MDPs that quickly learns a robust policy in the target environment by leveraging an abstract simulator. We have also proved a sample complexity bound, which does not scale with the size of the state space or the observation space. Finally, we demonstrated theoretical properties by empirical evaluation in two domains.

Our work raises many important questions. Our theoretical analysis depends on the assumption that the abstract simulator is deterministic and the the target environment is its perturbation. 
Is it possible to relax these requirements and still obtain sample bounds that do not scale with the state space size? One direction might be to leverage ``full simulators'' (including observations), and incorporate observation similarity between the simulator and the real environment, along with similarity in the latent dynamics. Another important question is how to extend the presented approach to continuous control problems in robotics, which are natural domains for sim-to-real transfer.

\bibliographystyle{plainnat}
\bibliography{ref}

\newpage
\onecolumn
\appendix
In the appendices, we include proofs, extended theoretical results as well as experimental details for all claims in the main paper, under the following structure.
\begin{itemize}
    \ifarxiv \else \item Appendix \ref{app:related_work} includes a more detailed discussion of related work.  \fi
    \item Appendix \ref{appendix:proof1} includes the proof for the analysis of~\pref{alg:robust_dp}.
    \item Appendix \ref{appendix:proof2} includes the proof for the analysis of~\pref{alg:trembling_hand_decoding}.
    \item Appendix \ref{appendix:stochastic_start} includes the algorithm and analysis for the case of stochastic initial states, based on the main algorithm. 
    \item Appendix \ref{appendix:experiment} includes more details of the experiments.
\end{itemize}

\section{Proofs for Algorithm \ref{alg:robust_dp}}
\label{appendix:proof1}
We introduce the definition of the (single-step) perturbed transition and reward set, following the definition of $\perturbclass{\Msim}{\eta}$. 
\begin{definition}[$\eta$-perturbation]
\label{def:perturb_rt_class}
We say that a transition and reward function pair $(R',T')$ at step $h$ is an \emph{$\eta$-perturbation} of $T_h, R_h$ if there exists a function $\xi:\Scal\times\Acal\to\Delta(\Acal)$ that satisfies $\xi(a\given s,a)\ge 1-\eta$ for all $s$, $a$,
and such that
\begin{align*}
&\textstyle
   T'(s'\given s,a)
   = \sum_{a' \in \Acal} T_h(s'\given s,a')\xi(a'\given s,a)
\\
&\textstyle
   R'(s,a)
   = \sum_{a' \in \Acal} R_h(s,a')\xi(a'\given s,a)
\end{align*}
With slight abusing of the notation, the set of all $\eta$-perturbations of $T_h, R_h$ is denoted $\perturbclass{(R_h, T_h)}{\eta}$.
\end{definition}
It is straightforward to see that $(R', T') \in \perturbclass{(R_h, T_h)}{\eta}$ if and only if there exists MDP $M' \in \perturbclass{M}{\eta}$ such that $(R', T')$ is the $h$-th step reward and transition functions in $M'$ and $(R_h, T_h)$ is the $h$-th step reward and transition functions in $M$.

Now we prove a lemma about the function $Q_h$ used in Algorithm \ref{alg:robust_dp}.
\begin{lemma} 
\label{lem:robust_be}
$$ (1- \perturbradius) \max_a \robustdpQ{h}(s,a) + \perturbradius \min_a \robustdpQ{h}(s,a) = \max_{a} \min_{ (R,T) \in \perturbclass{(\Rsim_h,\Tsim_h)}{\perturbradius}} \Expt_{ s' \sim T(\cdot|s,a) } [R(s,a) + \robustdpV{h+1}(s') ] $$
\end{lemma}
\begin{proof}
By the definition of $\perturbclass{(\Rsim_h,\Tsim_h)}{\perturbradius}$, we know that
\begin{align}
    & \min_{ (R,T) \in \perturbclass{(\Rsim_h,\Tsim_h)}{\perturbradius} } \Expt_{ s' \sim T(\cdot|s,a) } [R(s,a) + \robustdpV{h+1}(s') ] \\
    = & \min_{(R,T) \in \perturbclass{(\Rsim_h,\Tsim_h)}{\perturbradius} } R(s,a) + \sum_{s'} T(s'|s,a) \robustdpV{h+1}(s') \\
    = & \min_{\perturbprob(\cdot|s,a)} \sum_{a'} \left( \perturbprob(a'|s,a) \Rsim_h(s,a') + \sum_{s'} \perturbprob(a'|s,a) \Tsim_h (s'|s,a) \robustdpV{h+1}(s') \right)\\
    = & \min_{\perturbprob(\cdot|s,a)} \sum_{a'} \perturbprob(a'|s,a) \left( \Rsim_h (s,a') + \sum_{s'} \Tsim_h (s'|s,a) \robustdpV{h+1}(s') \right) \\
    = & \min_{\perturbprob(\cdot|s,a)} \sum_{a'} \perturbprob(a'|s,a) \robustdpQ{h}(s,a') \\
    & \ge (1-\perturbradius) \robustdpQ{h}(s,a) + \perturbradius \min_{a'} \robustdpQ{h}(s,a')
\end{align}
Thus we proved that 
\begin{align*}
    (1- \perturbradius) \robustdpQ{h}(s,a) + \perturbradius \min_{a'} \robustdpQ{h}(s,a') = \min_{  (R,T) \in \perturbclass{(\Rsim_h,\Tsim_h)}{\perturbradius} } \Expt_{ s' \sim T(\cdot|s,a) } [R(s,a) + \robustdpV{h+1}(s') ]
\end{align*}
The statement can be proved by taking the maximum over actions on both side.
\end{proof}

Now we prove the Theorem \ref{thm:robust_dp_value} as the main result for Algorithm \ref{alg:robust_dp}.

\begin{proof}
By Theorem 2.2 in \cite{iyengar2005robust}, since $Q_h$ and $V_h$ satisfy Lemma \ref{lem:robust_be}, we have that $\robustdpV{1}(s) = \sup_{\latentpolicy \in \latentpolicyset} \min_{M \in \perturbclass{\Msim}{\perturbradius}} V^{\latentpolicy}_{M,1}(s)$. The perturbation we defined consists of independent perturbation at different time step $h$, therefore it satisfy the ``rectangularity'' condition in \citet{iyengar2005robust} \footnote{This assumptions requires the choice of perturbation in one time step cannot limit the choice of perturbation in other time steps. We refer to \citet{iyengar2005robust} for the formal definition. }. From \cite{iyengar2005robust}, the robust policy is given by 
\begin{align}
    \robustlatentpolicy_h(s) =& \argmax_{a} \min_{  (R,T) \in \perturbclass{(\Rsim_h,\Tsim_h)}{\perturbradius} } \Expt_{ s' \sim T(\cdot|s,a) } [R(s,a) + \robustdpV{h+1}(s') ] \\
    =& \argmax_{a} \left( \robustdpQ{h}(s,a) + \perturbradius \min_{a'} \robustdpQ{h}(s,a')  \right) \\
    =& \argmax_{a} \robustdpQ{h}(s,a)
\end{align}
So $\robustlatentpolicy$ is returned by Algorithm \ref{alg:robust_dp}.
\end{proof}

\subsection{An alternative Proof for Algorithm \ref{alg:robust_dp} from the First Principle}
The previous proof relies on a more general results from \citet{iyengar2005robust} in a more complicated form. However the proof can be simplified since our perturbation set takes a particular linear structure. So for completeness and readability, we also include another complete proof derived from the first principle. We acknowledge that the proof is heavily inspired by previous work in robust dynamics programming \citep{bagnell2001solving,iyengar2005robust}.

Let $\latentpolicyset$ be the set of all mappings from $[H] \times \Scal$ to $\Delta(\Acal)$. It is straightforward to verify the following inequality for value function.
\begin{proposition}
\label{prop:minmaxvalue}
$\forall h \in [H], s \in \Scal$
\begin{align*}
    \max_{\latentpolicy \in \latentpolicyset } \min_{ M \in \perturbclass{\Msim}{\eta} } V^{\latentpolicy}_{M,h}(s) \le \min_{M \in \perturbclass{\Msim}{\eta} } \max_{\latentpolicy \in \latentpolicyset} V^{\pi}_{M,h}(s) \defeq \min_{M \in \perturbclass{\Msim}{\eta}} V^{\star}_{M,h}(s)
\end{align*}
\end{proposition}
Given this, we are going to explain the high level structure of the proof. For $\robustdpV{h}$, we are going to show that
\begin{align}
    \robustdpV{h}(s) \ge \min_{M \in \perturbclass{\Msim}{\eta}} V^{\star}_{M,h}(s) \label{eq:robust_dp_proof_step1} \\
    \robustdpV{h}(s) \le \min_{ M \in \perturbclass{\Msim}{\eta} } V^{\robustdppolicy}_{M,h}(s), \label{eq:robust_dp_proof_step2}
\end{align}
for some $\robustdppolicy$.

Then by Proposition \ref{prop:minmaxvalue}, we will have that
\begin{align}
    \robustdpV{h}(s) \le \min_{ M \in \perturbclass{\Msim}{\eta} } V^{\robustdppolicy}_{M,h}(s) \le \max_{\latentpolicy \in \latentpolicyset } \min_{ M \in \perturbclass{\Msim}{\eta} } V^{\latentpolicy}_{M,h}(s) \le \min_{M \in \perturbclass{\Msim}{\eta}} V^{\star}_{M,h}(s) \le  \robustdpV{h}(s)
\end{align}
Then we will have that all the inequalities must be equality and the $\robustdppolicy$ must be the robust policy $\rho$.

Now we will show that Equation \ref{eq:robust_dp_proof_step1} and Equation \ref{eq:robust_dp_proof_step2} holds for the $\robustdppolicy$ which is the output from Algorithm \ref{alg:robust_dp}, in the following two lemmas.
 
\begin{lemma}
For any $s \in \Scal$ and $h \in [H]$, $\robustdpV{h} (s) \ge \min_{ M \in \perturbclass{\Msim}{\eta} } V^{\star}_{M,h}(s) $
\end{lemma}
\begin{proof}
    For any $h \in [H]$, we will construct a $M \in \perturbclass{\Msim}{\eta}$ such that $\forall s, \robustdpV{h} (s) = V^{\star}_{M,h}(s)$. Then for any $s \in \Scal$, $\robustdpV{h} (s) = V^{\star}_{M,h}(s) \ge \min_{ M \in \perturbclass{\Msim}{\eta} } V^{\star}_{M,h}(s)$ and we prove this lemma. Now we prove the following statement by induction on $h$. For any $h \in [H+1]$, $\exists M \in \perturbclass{\Msim}{\eta}$ such that $\forall s \in \Scal, \robustdpV{h}(s) = V^{\star}_{M,h}(s)$.
    
    For the basement of induction, $\robustdpV{H+1}(s) = V^{\star}_{M,h+1}(s) = 0$ for any $M$. Thus the induction statement holds for $h = H+1$. 
    
    Second, let us assume that the induction statement holds for $h+1$, where $h \in [H]$: $\forall s, \robustdpV{h+1}(s) = V^{\star}_{M_{h+1},h+1}(s)$ for some $M_{h+1} \in \perturbclass{\Msim}{\eta}$. Notice that this statement is for the value function at $h+1$ step and holds for $M_{h+1}$, thus for any $M \in \perturbclass{\Msim}{\eta}$ that shares the reward and transition functions on and after step $h+1$, the statement also holds. This is because in our definition of episodic MDP the perturbation class the reward and transition functions at different steps are independent.

    Now we are going to construct $M_h$ such that $\robustdpV{h}(s) = V^{\star}_{M_{h},h}(s)$. More specifically, we will only construct the reward and transition functions at $h$-th step: $R_h$ and $T_h$, and concatenate it with other component in $M_{h+1}$.
    
    For any $s$, let $\underline{a}$ be $\argmin_{a \in \Acal} \robustdpQ{h}(s,a)$. Given the $s$ and $\underline{a}$, construct $T_h(s'|s,a)$ and $R_h(s,a)$ for any $s,a,s'$ as the following.
    \begin{align}
        T_h(s'|s,a) &= (1-\eta)\Tsim_h(s'|s,a) + \eta \Tsim_h (s'|s,\underline{a}) \\
        R_h(s,a) &= (1-\eta)\Rsim_h(s,a) + \eta \Rsim_h (s,\underline{a})
    \end{align}
    Then we finish the proof of statement for $h$ by
    \begin{align}
        \robustdpV{h}(s) = & (1-\eta) \max_a \robustdpQ{h}(s,a) + \eta \min_a \robustdpQ{h}(s,a) \\
        = & (1-\eta) \max_{a} \left( \Rsim_h(s,a) + \sum_{s'} \Tsim_h(s'|s,a) \robustdpV{h+1}(s') \right) + \eta \left( \Rsim_h(s, \underline{a} ) + \sum_{s'} \Tsim_h(s'|s,\underline{a}) \robustdpV{h+1}(s') \right) \\
        = & \max_{a} \left(  R_h(s,a) + (1-\eta) \sum_{s'} \Tsim_h(s'|s,a) \robustdpV{h+1}(s') + \eta \sum_{s'} \Tsim_h(s'|s,\underline{a}) \robustdpV{h+1}(s') \right) \\
        = & \max_{a} \left( R_h(s,a) +  \sum_{s'} \left( (1-\eta) \Tsim_h(s'|s,a) + \eta \Tsim_h(s'|s,\underline{a}) \right)\robustdpV{h+1}(s') \right) \\
        = & \max_{a} \left( R_h(s,a) + \sum_{s'} T_h(s'|s,a) \robustdpV{h+1}(s') \right) \\
        = & \max_{a} \left( R_h(s,a) + \sum_{s'} T_h(s'|s,a) V^{\star}_{M_{h+1},h+1}(s) \right) \\
        = & V^{\star}_{M_{h},h}(s)
    \end{align}
    By induction, we proved that there exist a $M \in \perturbclass{\Msim}{\eta}$ such that $\robustdpV{h}(s) = V^{\star}_{M,h}(s)$ for any $s \in \Scal, h \in [H+1]$. Thus we finish the proof of $\robustdpV{h}(s) \ge \min_{ M \in \perturbclass{\Msim}{\eta} } V^{\star}_{M,h}(s) $.
\end{proof}

\begin{lemma}
For the policy $\robustdppolicy$ computed from Algorithm \ref{alg:robust_dp}, $\forall h \in [H], s \in \Scal, \min_{M \in \perturbclass{\Msim}{\eta}} V^{\robustdppolicy}_{M,h}(s) \ge \robustdpV{h}(s)$
\end{lemma}
\begin{proof}
    We prove it by induction on $h$ from $H+1$ to $1$. In the base case ($h=H+1$) we have, $\robustdpV{H+1}(s) = 0 = V^{\robustdppolicy}_{M,H+1}(s)$ for all $s \in \Scal$ and any $M$. This implies: $\forall s, \robustdpV{H+1}(s) \le \min_{M \in \perturbclass{\Msim}{\eta}} V^{\robustdppolicy}_{M,H+1}(s)$. 
    
    We now consider the general case. We assume that $\forall s \in \Scal, \robustdpV{h+1}(s) \le \min_{M \in \perturbclass{\Msim}{\eta}} V^{\robustdppolicy}_{M,h+1}(s)$. Fix any $s \in \Scal$ and any $M \in \perturbclass{\Msim}{\eta}$, we are going to show that $\robustdpV{h}(s) \le V^{\robustdppolicy}_{M,h}(s)$. Let $R_h$ and $T_h$ be its reward and transition function at $h$-th step, and $\xi_h$ be the perturbation variables. We have $1- \xi_h(a\given s,a) \le \eta$ by the definition of $M \in \perturbclass{\Msim}{\eta}$. So $1-\xi_h(\robustdppolicy(s)\given s,\robustdppolicy(s))  \le \eta$.
    
    \begingroup
\allowdisplaybreaks
    \begin{align}
        \robustdpV{h}(s) = & (1-\eta) \max_a \robustdpQ{h}(s,a) + \eta \min_a \robustdpQ{h}(s,a) \\
        &= (1-\eta) \left( \Rsim_h(s,\robustdppolicy(s)) +  \sum_{s'} \Tsim_h(s'|s,\robustdppolicy(s)) \robustdpV{h+1}(s') \right) \\
        &~~~~ + \eta \left( \Rsim_h(s,\underline{a}) +  \sum_{s'} \Tsim_h(s'|s,\underline{a}) \robustdpV{h+1}(s') \right) \\
        &\le \xi_h(\robustdppolicy(s)\given s,\robustdppolicy(s))  \left( \Rsim_h(s,\robustdppolicy(s)) +  \sum_{s'} \Tsim_h(s'|s,\robustdppolicy(s)) \robustdpV{h+1}(s') \right) \nonumber \\
        &~~~~ + \left( \sum_{a \neq \robustdppolicy(s) }   \xi_h(a \given s,\robustdppolicy(s))  \right) \left( \Rsim_h(s,\underline{a}) +  \sum_{s'} \Tsim_h(s'|s,\underline{a}) \robustdpV{h+1}(s') \right) \\
        &\le \xi_h(\robustdppolicy(s)\given s,\robustdppolicy(s))  \left( \Rsim_h(s,\robustdppolicy(s)) +  \sum_{s'} \Tsim_h(s'|s,\robustdppolicy(s)) \robustdpV{h+1}(s') \right) \nonumber \\
        &~~~~ + \sum_{a \neq \robustdppolicy(s) }  \xi_h(a \given s,\robustdppolicy(s)) \left(  \Rsim_h(s,a) +  \sum_{s'} \Tsim_h(s'|s,a) \robustdpV{h+1}(s') \right)  \\
        &= \sum_{a} \xi_h(a \given s,\robustdppolicy(s)) \left( \Rsim_h(s,a) + \sum_{s'} \Tsim_h(s'|s,a) \robustdpV{h+1}(s') \right)\\
        &= R_h(s,\robustdppolicy(s)) + \sum_{s'} T_h(s'|s,\robustdppolicy(s)) \robustdpV{h+1}(s') \\
        &\le R_h(s,\robustdppolicy(s)) + \sum_{s'} T_h(s'|s,\robustdppolicy(s)) \min_{M' \in \perturbclass{\Msim}{\eta}} V^{\robustdppolicy}_{M',h+1}(s')  \\
        &\le R_h(s,\robustdppolicy(s)) + \sum_{s'} T_h(s'|s,\robustdppolicy(s)) V^{\robustdppolicy}_{M,h+1}(s') 
        = V^{\robustdppolicy}_{M,h}(s)
    \end{align}
     \endgroup
    Since this is for any $M \in \perturbclass{\Msim}{\eta} $, we proved that $\robustdpV{h}(s) \le \min_{M \in \perturbclass{\Msim}{\eta}} V^\pi_{M,h}(s)$.
\end{proof}

Now we prove the Theorem \ref{thm:robust_dp_value} as the main result for Algorithm \ref{alg:robust_dp}.

\begin{proof}
We have proved Equation \ref{eq:robust_dp_proof_step1} and Equation \ref{eq:robust_dp_proof_step2} holds for the $\robustdppolicy$ which is the output from Algorithm \ref{alg:robust_dp}, in previous two lemmas. Now combine this with the inequality of min max values, we have that
\begin{align}
    \robustdpV{h}(s) \le \min_{ M \in \perturbclass{\Msim}{\eta} } V^{\robustdppolicy}_{M,h}(s) \le \max_{\latentpolicy \in \latentpolicyset } \min_{ M \in \perturbclass{\Msim}{\eta} } V^{\latentpolicy}_{M,h}(s) \le \min_{M \in \perturbclass{\Msim}{\eta}} V^{\star}_{M,h}(s) \le \robustdpV{h}(s)
\end{align}
So we will have that all the inequalities must be equality and the $\robustdppolicy$ output by Algorithm \ref{alg:robust_dp} must be the robust policy $\rho$ defined in Equation \ref{eqn:robust_policy_main_paper}.
\end{proof}

\section{Proofs for Algorithm \ref{alg:trembling_hand_decoding}}
\label{appendix:proof2}
In this section, we are going to prove a stronger version of Theorem \ref{thm:robust_tranfer} stated below. We relax the assumption such that the initial state in the true environment has a small probability when it is not equals to the $\sinit$
\begin{theorem}
\label{thm:robust_tranfer_stronger}

Let $\Msim$ be a deterministic abstract simulator. Let $\blockMreal$ be a
target environment for which $\Mreal$ is an $\eta$-perturbation of $\Msim$
for some $\eta<0.5$. Let  $\epsilon_0$ be the probability that $s_0 \ne \sinit$ in $M$. Let $\Fcal$ be a class of functions realizable with respect to $\blockMreal$. Then for any $\epsilon>0$ and $\delta \in (0, 1)$, \pref{alg:trembling_hand_decoding} with oracle access to
$\blockMreal$, optimization-oracle access to $\Fcal$, and inputs $\Msim$, $\eta$, $\epsilon$, and $\delta$,
executes $n=\Ocal\bigl(\frac{H^4 |\Acal|^3\ln(|\Fcal|/\delta)}{\epsilon(1-2\perturbradius) ^2}\bigr)$ \footnote{Notice that here the dependency on horizon is $H^4$ instead of $H^3$ in the main paper. We made a mistake in the statement of this theorem in the main paper.} episodes and returns a practicable policy $\policy$ that with probability at least $1-\delta$ satisfies
\[
V^{\pi}_{\blockMreal} \ge
V^{\robustlatentpolicy \circ \phi^\star}_{\blockMreal} - \epsilon - H\epsilon_0.
\]
\end{theorem}

The proof to Theorem \ref{thm:robust_tranfer_stronger} has three steps. First we show optimal classifier of the classification problem in~\pref{line:alg_classification} and the accuracy of learned ERM classifier. Second we show the accuracy of the learned action decoder $\actionpredictor_{h+1}(x_h,x_{h+1})$ under the roll-in distribution of $x_h$. Last the error bound in the theorem can be proved by bounding the union probability of failing to predict the action in each steps. Now we give some lemmas and proofs in these three steps in the following three subsections, followed by the proof of the final sample complexity results. 

\subsection{Accuracy of action classification}
First we define some notation that will be used in the proofs for Algorithm \ref{alg:trembling_hand_decoding}. Let the uniform distribution over action space be denoted $\unf(\Acal)$, $\Prob_h(x)$ be the distribution where the $x_h$ is sampled from in Line \ref{line:potas-dataset-end} in Algorithm \ref{alg:trembling_hand_decoding}, and $\Prob_h(x,a,x') \defeq \Prob_h(x) \times \unf(\Acal) \times \Treal_h(x'|x,a)$ denote the joint distribution where the $x_h, a_h, x_{h+1}$ are sampled from in Line \ref{line:potas-dataset-end}.

Now we show the form of conditional distribution of action given two observations, $\Prob_h(a \given x,x')$, under the joint distribution $\Prob_h(x,a,x')$. Moreover, we proved the form of $\Prob_h(a \given x,x')$ for any roll-in distribution of the first observation $x$.
\begin{lemma}
\label{lem:posterior_action_distribution}
Given any prior distribution $\Prob(x)$ of $x$, the uniform conditional distribution of $a$ given $x$, and transition $\Treal_h(x'|x,a) \defeq \Treal_h(\inverseobsmap^\star(x')|\inverseobsmap^\star(x),a)$, the posterior distribution of $a$ given $x,x'$ is:
\begin{align}
    \Prob_{h}(a \given x,x') = \frac{\Treal_h( x'|x,a )}{\sum_{a'} \Treal_h( x'|x,a' ) } = \frac{\Treal_h( \inverseobsmap^\star(x')|\inverseobsmap^\star(x),a )}{\sum_{a'} \Treal_h( \inverseobsmap^\star(x')|\inverseobsmap^\star(x),a' ) }
\end{align}
\end{lemma}
\begin{proof}
\begin{align}
    \Prob_{h}(a \given x,x') = \frac{\Prob_h(x,a,x')}{\sum_{a} \Prob_h(x,a,x') } =& \frac{  \Prob(x) \unf(a) \Treal_h(x'|x,a) }{\sum_{a'}  \Prob(x) \unf(a') \Treal_h(x'|x,a')} \\
    =& \frac{  \nicefrac{\Prob(x)}{|\Acal|}\Treal_h(x'|x,a) }{\nicefrac{\Prob(x)}{|\Acal|} \sum_{a'} \Treal_h(x'|x,a')} \\
    =& \frac{ \Treal_h(x'|x,a) }{ \sum_{a'} \Treal_h(x'|x,a')} \\
    =& \frac{\Treal_h(\inverseobsmap^\star(x') \mid \inverseobsmap^\star(x), a)}{\sum_{a' \in \Acal}\Treal_h(\inverseobsmap^\star(x') \mid \inverseobsmap^\star(x), a')}
\end{align}
\end{proof}

Next, we show that the maximizer of the expected log likelihood is the conditional distribution of the action given observations $\Prob_h (a \given x,x')$.
\begin{lemma}
\label{lem:optimal_action_classifier}
For $(x,a,x') \sim \Prob_h(x,a,x')$, and function class $\Fcal \subset \Xcal^2 \mapsto \Delta(\Acal) $ satisfying Assumption \ref{asm:realizable} (realizability):
\begin{align}
    \argmax_{f \in \Fcal} \Expt_{x,a,x'}[ \ln f(a \given x,x') ] = \Prob_{h}(a \given x,x') = \frac{\Treal_h( x'|x,a )}{\sum_{a'} \Treal_h( x'|x,a' ) }
\end{align}
\end{lemma}
\begin{proof}
    For any $f \in \Fcal$, and any $(x,x')$, we are going to show that 
    \begin{align}
        \Expt_{a \sim \Prob_h(\cdot|x,x') }[ \ln f(a \given x,x') ] \le \Expt_{a \sim \Prob_h(\cdot|x,x')}[ \ln \Prob_{h}(a \given x,x') ]
    \end{align}
    Fixing $(x,x')$, $f(a \given x,x')$ is a probability mass function over $\Acal$. Let it be denoted $\mathrm{Q}(a)$. According to Jensen's inequality,
    \begin{align}
        \Expt_{a\sim \Prob}[ \ln \mathrm{Q}(a) ] \le \Expt_{a \sim \Prob}[ \ln \Prob(a) ],
    \end{align}
    for any distribution $\Prob$ over $\Acal$. Let $\Prob(a) = \Prob_{h}(a \given x,x')$. We have that 
    \begin{align}
        \Expt_{a \sim \Prob_h(\cdot|x,x')}[ \ln f(a \given x,x') ] \le \Expt_{a \sim \Prob_h(\cdot|x,x')}[ \ln \Prob_{h}(a \given x,x') ]
    \end{align}
    Then 
    \begin{align}
        \Expt_{x,a,x'}[ \ln f(a \given x,x') ] = \Expt_{x,x'}\Expt_{a \sim \Prob_h(\cdot|x,x')}[ \ln f(a \given x,x') ] \\
        \le \Expt_{x,x'}\Expt_{a \sim \Prob_h(\cdot|x,x')} [ \ln \Prob_{h}(a \given x,x') ] = \Expt_{x,a,x'}[ \Prob_{h}(a \given x,x') ]
    \end{align}
    Assumption \ref{asm:realizable} of $\Prob_{\unf}(a \given x,x') \in \Fcal$ finished the proof.
\end{proof}

\begin{align}
    f^\star_h(a,x,x') &\defeq \argmax_{f \in \Fcal} \Expt_{x,a,x'}[ \ln f(a,x,x') ] = \textstyle\frac{\Treal_h( \inverseobsmap^\star(x')|\inverseobsmap^\star(x),a )}{\sum_a' \Treal_h( \inverseobsmap^\star(x')|\inverseobsmap^\star(x),a' ) }
\end{align}
Given that , our empirical maximizer
$ f_h (a,x,x') = \argmin_{f \in \Fcal} \Expt_\Dcal [\ln f(a,x,x')] $ 

Now we can show the error bound of empirical log-likelihood maximizer.
\begin{theorem}[Theorem 21 in \citet{agarwal2020flambe} with I.I.D. data]
\label{thm:flambe_tv_error}
Let $\Dcal_h$ be a data set with $\ndecode$ transitions (x,a,x') sampled i.i.d. from $\Prob_h(x,a,x')$. Let $f_h(a \given x, x')$ and $ f^\star_h(a \given x, x')$ denote the maximizers of empirical log-likelihood and expected log-likelihood respectively:
\begin{align}
    f^\star_h(a \given x,x') &\defeq \argmax_{f \in \Fcal} \Expt_{x,a,x' \sim \Prob_h }[ \ln f(a,x,x') ] = \textstyle\frac{\Treal_h( \inverseobsmap^\star(x')|\inverseobsmap^\star(x),a )}{\sum_a' \Treal_h( \inverseobsmap^\star(x')|\inverseobsmap^\star(x),a' ) } \\
    f_h(a \given x, x') &\defeq \argmax_{f \in \Fcal} \sum_{(x_{h}, a_{h}, x_{h+1})\in \Dcal_h} \ln f(a_h \given x_h, x_{h+1})
\end{align}
Then for any $\delta \in (0, 1)$, with probability at least $1-\delta$ we have that: 
\begin{align}
    \Expt_{x,x' \sim \Prob_h} \left[\left\|f_h(a, x, x') - f^\star_h(a, x, x')\right\|_{\text{TV}}^2\right] \le \textstyle \frac{2\ln(|\Fcal|/\delta)}{\ndecode}.
\end{align}
\end{theorem}

\subsection{One-step accuracy of action decoder}
In the last subsection, we have showed the learned function $f_h$ approaches the posterior distribution of action at a rate of $\nicefrac{1}{\ndecode}$. In Algorithm \ref{alg:trembling_hand_decoding}, the learned function $f_h$ is used to build the state decoder $\actionpredictor_h$. This state decoding process relies on identified the correct ``shadow'' actions. 

In this section we first show that there is a separation between the correct shadow action's probability and random actions' probabilities, in the posterior distribution of action. This relies on the transition dynamics in target environment lies in $\perturbclass{\Tsim}{\perturbradius}$. Then we show the action decoding accuracy. The idea to bound the 1-step action decoding accuracy is to apply the Markov's inequality on the event that the $\argmax$ action of $f_h$ is wrong. 

First, we define the set of shadow actions given two successive states.
\begin{definition}[Shadow action set] 
For any states pair $(s,s')$, we define a set of actions $\Aset_h(s,s') = \{ a \in \Acal: \Tsim_h(s'|s,a) = 1 \}$ and $\Acompset_h(s,s') = \Acal \backslash \Aset_h (s,s')$.

Since the observation emission function $q^\star$ maps different $s$ to disjoint observation subspace, the shadow action sets can also be defined on the corresponding observation pairs $x,x'$. We define $\Aset_h(x,x') = \{ a \in \Acal: \Tsim_h( \inverseobsmap^\star(x')|\inverseobsmap^\star(x),a) = 1 \}$, and similarly for  $\Acompset_h(x,x')$
\end{definition}
By definition, we have that $\Aset_h(x,x') = \Aset_h(s,s')$ if $x \sim q^\star(\cdot|s)$ and $x' \sim  q^\star(\cdot|s')$. So later we may use these two notations interchangeably for convenience.

Next we prove the accuracy result of decoder $\actionpredictor$ under the joint distribution with \emph{any} practicable policy, but learned from the dataset from \emph{uniform} action distribution. We use $\ind$ as an indicator function of random events.

\begin{lemma}[Accuracy of decoder] 
\label{lem:1step_decoding_accuracy}
For any practicable policy $\policy$, let $\Prob_{h,\policy}(x,a,x') \defeq \Prob_h(x) \policy(a|x) \Treal(x'|x,a) $ be the joint distribution of $\Prob_h(x)$, policy $\policy$ and transition function $\Treal$.

Let $\Dcal_h$ be a data set with $\ndecode$ transitions (x,a,x') sampled i.i.d. from $\Prob_h(x,a,x')$. For any $h \in [H]$ and $\delta \in (0,1)$, and any practicable policy $\pi$, with probability at least $1-\delta$, we have that
\begin{align}
    \Expt_{\Prob_{h,\policy}}\left[\ind (\actionpredictor_h(x,x') \in \Aset_h(x,x'))\right] \ge 1-\frac{8h|\Acal|^3\ln(|\Fcal|/\delta)}{\ndecode(1-2\perturbradius)^2}
\end{align}
\end{lemma}

\begin{proof}
For any $a_1 \in \Aset_h(x,x')$, and $a_2 \in \Acompset_h(x,x')$, 
\begin{align}
     f^\star_h(a_1 \given x,x') -  f^\star_h(a_2 \given x,x') &= \Prob_{h}(a_1|x,x') - \Prob_{h}(a_2|x,x') \\
     &= \frac{\Treal_h(\inverseobsmap^\star(x') \mid \inverseobsmap^\star(x), a_1) - \Treal_h(\inverseobsmap^\star(x') \mid \inverseobsmap^\star(x), a_2)}{\sum_{a' \in \Acal}\Treal_h(\inverseobsmap^\star(x') \mid \inverseobsmap^\star(x), a')} \\
     &\ge \frac{(1-\perturbradius) - \perturbradius}{\sum_{a' \in \Acal}\Treal_h(\inverseobsmap^\star(x') \mid \inverseobsmap^\star(x), a')} \ge \frac{(1-2\perturbradius)}{|\Acal|}
\end{align}

Given that the gap is $\nicefrac{(1-2\perturbradius)}{|\Acal|}$, we have that for any fixed $x,x'$, if $\| f_h(\cdot,x,x') - f^\star_h(\cdot,x,x') \|_{\text{TV}} < \nicefrac{(1-2\perturbradius)}{2|\Acal|}$, then $\argmax f_h(\cdot,x,x') \in \Aset_h(x,x')$. 

Recall that $\actionpredictor_h(x,x') = \argmax_{a \in \Acal} f_h(a \given x,x')$. We have
\begin{align}
\Expt_{\Prob_h } [\ind(\actionpredictor_h(x,x') \notin \Aset_h(x,x')) ] =& \Pr \left(\actionpredictor_h(x,x') \notin \Aset_h(x,x') \right) \\
  \le & \Pr\left( \|f^\star_h(\cdot \given x,x')-f_h(\cdot \given x,x')\|_{\text{TV}}  \ge \frac{(1-2\perturbradius)}{2|\Acal|} \right) \\
    = & \Pr \left( \|f^\star_h(\cdot \given x,x')-f_h(\cdot \given x,x')\|_{\text{TV}}^2  \ge \frac{(1-2\perturbradius)^2}{4|\Acal|^2}\right) \\
    \le& \frac{\Expt \left[  \|f^\star_h(\cdot \given x,x')-f_h(\cdot,\given x,x')\|_{\text{TV}}^2 \right] }{ \nicefrac{(1-2\perturbradius)^2}{4|\Acal|^2}} \\
    \le & \frac{8|\Acal|^2\ln(|\Fcal|/\delta)}{\ndecode(1-2\perturbradius)^2}
\end{align}
The second to last step follows from Markov's inequality, and the last step follows from Theorem \ref{thm:flambe_tv_error}. 

Notice that this is the error under distribution of uniform action $\Prob_h(x,a,x') := \Prob(x) \circ \unf \circ \Treal_h$. For any practicable policy $\pi$, since $ \frac{\Prob_{h,\pi}(x,a,x')}{\Prob_h(x,a,x')} = \frac{\pi(a|x)}{1/|\Acal| }\le |\Acal|$,
\begin{align}
    \Expt_{\Prob_{h, \policy}}\left[\ind(\actionpredictor_h(x,x') \notin \Aset_h(x,x'))\right] \le \frac{8|\Acal|^3\ln(|\Fcal|/\delta)}{\ndecode(1-2\perturbradius)^2}
\end{align}
Then taking the complement of the event in the indicator function finished the proof.
\end{proof}

\subsection{Analysis of the sample complexity}

Now we are going prove that the learned policy recover the input latent policy with high probability.
\begin{lemma}
\label{lem:nstep_decode_accuracy}
For any $h \in [H]$, let $\mathbf{x}_{1:h+1} \defeq x_1, x_2, \dots, x_{h+1}$ and $\mathbf{a}_{1:h} \defeq a_1, a_2, \dots, a_h$ be the state and action sequence generated from $(\policy_1, \dots, \policy_h)$ output by Algorithm \ref{alg:trembling_hand_decoding}. Given the high probability event in Theorem \ref{thm:flambe_tv_error},  we have that for any $h \in [H]$
\begin{align}
    \Expt_{\policy_{1:h} } \left[ \sum_{k=1}^{h}  \ind \{ a_k \neq \robustlatentpolicy(\inverseobsmap(x_k)) \} \right] \le \frac{8h^2|\Acal|^3\ln(|\Fcal|/\delta)}{\ndecode(1-2\perturbradius)^2} + \epsilon_0
\end{align}
\end{lemma}

\begin{proof}
If the action decoder $\actionpredictor_k$ is correct, i.e. $\actionpredictor_k(x_k,x_{k+1}) \in \Aset_k(x_k,x_{k+1})$ for any $k \in [h]$. Then if the initial state is $\sinit$, we have that for any $k \in [h]$, $\inverseobsmap_k(\mathbf{x}_{1:k}) = \inverseobsmap^\star (x_{k+1})$ due to the determinism of the abstract simulator. Then for any $k \in [h]$,
\begin{align}
    a_k = \policy_k(\mathbf{x}_{1:k}) = \robustlatentpolicy( \inverseobsmap_k(\mathbf{x}_{1:k}) ) = \robustlatentpolicy(\inverseobsmap^\star(x_k))
\end{align}
That means if we have $a_k \neq \robustlatentpolicy(\inverseobsmap^\star(x_k))$, we must have $\actionpredictor_k(x_k,x_{k+1}) \notin \Aset_k(x_k,x_{k+1})$ for some $j \le k$, or the initial state is not $\sinit$. So under the condition that the initial state is $\sinit$, 
\begin{align}
    \ind \{ a_k \neq \robustlatentpolicy(\inverseobsmap^\star(x_k)) \} \le \sum_{j=1}^k \ind \{ \actionpredictor_j(x_j,x_{j+1}) \notin \Aset_j(x_j,x_{j+1}) \}
\end{align}
Thus we have
\begin{align}
    &\Expt_{\policy_{1:h} } \left[ \sum_{k=1}^{h}  \ind \{ a_k \neq \robustlatentpolicy(\inverseobsmap^\star(x_k)) \} \right] \\\
    = & \Expt_{\policy_{1:h} } \left[ \sum_{k=1}^{h} \sum_{j=1}^k \ind \{ \actionpredictor_j(x_j,x_{j+1}) \notin \Aset_j(x_j,x_{j+1}) \} \right] + \epsilon_0 \\
    \le & \Expt_{\policy_{1:h} } \left[ h \sum_{j=1}^{h} \ind \{ \actionpredictor_j(x_j,x_{j+1}) \notin \Aset_j(x_j,x_{j+1}) \} \right] + \epsilon_0 \\
    \le & \frac{8h^2|\Acal|^3\ln(|\Fcal|/\delta)}{\ndecode(1-2\perturbradius)^2} + \epsilon_0 \tag{Lemma \ref{lem:1step_decoding_accuracy}}
\end{align}
\end{proof}
This immediately gives the following theorem by letting $h = H$ and bounding the value gap for non-optimal actions by $H$. The proof of Theorem \ref{thm:robust_tranfer} follows from this.
\begin{proof}
We first bound the value gap using preivous lemma and the Performance Difference Lemma \citep{kakade2003sample}.
\begin{align}
    V_{\blockMreal}^{\policy} - v_{\blockMreal}^{\robustlatentpolicy \circ \inverseobsmap^\star} \le& \Expt_{x_h,a_h \sim \policy} \left[ \sum_{h=1}^{H} Q^{\robustlatentpolicy\circ \inverseobsmap^\star}_{\blockMreal,h}(x_h,a_h) - V^{\robustlatentpolicy\circ \inverseobsmap^\star}_{\blockMreal, h}(x_h) \right] \tag{Performance Difference} \\
    = & \Expt_{s_h,a_h \sim \policy} \left[ \sum_{h=1}^{H} Q_{\Mreal,h}^{\robustlatentpolicy}(s_h,a_h) - V_{\Mreal,h}^{\robustlatentpolicy}(s_h) \right] \tag{Block MDP} \\
    \le& \Expt_{s_h,a_h \sim \policy} \left[ \sum_{h=1}^{H} H \ind\{ a_h \neq \robustlatentpolicy(\inverseobsmap(x_h)) \} \right] \tag{Any $s_h, a_h$, $Q^{\robustlatentpolicy}_{h}(s_h,a_h) \in [0, H]$} \\
    \le& \frac{8H^3 |\Acal|^3 \ln(|\Fcal|/\delta)}{\ndecode(1-2\perturbradius) ^2} + H\epsilon_0 \tag{Lemma \ref{lem:nstep_decode_accuracy} for $h=H$} \\
    =& \frac{8H^4 |\Acal|^3 \ln(|\Fcal|/\delta)}{n(1-2\perturbradius) ^2} + H\epsilon_0 \tag{$n = H\ndecode$}
\end{align}
Finally, we can solve the sample complexity by denote the value gap $\epsilon$, and finish the proof.
\end{proof}

\section{Analysis with Stochastic Initial State}
\label{appendix:stochastic_start}

We can extend $\algname$ to a more general setting where initial state can be stochastically chosen, instead of being deterministic. This setting captures problems such as navigation in a set of house simulators, where dynamics of each house simulator can be deterministic but the choice of initial state, i.e., choice of current house and position of the agent inside the house, can be stochastically chosen.

As~\pref{alg:robust_dp} does not rely on the deterministic initial state, therefore, we only need to show that we can extend~\pref{alg:trembling_hand_decoding} to the stochastic initial state setting, and find the robust policy learned by~\pref{alg:robust_dp}.

First, we introduce the difference in problem settings and our main results under this setting formally. We assume that in both abstract simulator $\Msim$ and the target environment $\blockMreal$, the initial states are sampled from the same distribution $\mu$ over a finite set of initial states $\Scal_1$ of size $|\Scal_1| = N$.

We make an assumption that each initial state occurs with a reasonable probability and has a different transition dynamics.
\begin{assumption}\label{assum:minimum_prob}
(Conditions on initial state) For all initial states $s \in \Scal_1$, we assume $\mu(s) \ge \mumin$ for some $\mumin \in (0, 1]$. Further, there exists a margin $\Gamma > 0$ such that for any two initial states $s, \tilde{s} \in \Scal_1$ we have:
\begin{equation*}
    \|T(\cdot \given s, a) - T(\cdot \given \tilde{s}, a)\|_{{\tt TV}}  \ge \Gamma.
\end{equation*}
\end{assumption}
Informally, this assumption is required so that we can visit each initial state sufficiently, and use the margin assumption to learn an accurate decoder to cluster initial states. However, note that we cannot directly cluster in the observation space since we do not want to make any additional structural assumptions on it. In contrast, we will use a function-approximation approach where we interact with the observation via a function class.

\begin{algorithm}[ht]
\caption{$\algname$ with multiple initial states}
\label{alg:stochastic_start}
\begin{algorithmic}[1]
\State \textbf{Input:} An approximate clustering function of initial states $\hat{\phi}$, number of initial state $N$, error bound $\epsilon$, failure probability $\delta$.
\State Initialize counter ${\tt step}(i) \leftarrow 0$ for $i \in [N]$
\State Initialize state map ${\tt map}(i) \leftarrow -1$ for $i \in [N]$
\For{episode $e=1, \dots$} 
    \State For the initial observation $x$, decode the state by $\hat{\phi}(x)$.
    \If {${\tt map}(\hat{\phi}(x)) < 0$}
    \State ${\tt map}(\hat{\phi}(x)) \leftarrow$ InitialStateTest$(\hat{\phi}(x), {\tt step}(\hat{\phi}(x)), N ,\epsilon, \delta)$
    \State ${\tt step}(\hat{\phi}(x)) \leftarrow {\tt step}(\hat{\phi}(x)) + 1$
    \Else
    \State Run $\algname$ on ${\tt map}(\hat{\phi}(x))$
    \EndIf
\EndFor
\end{algorithmic}
\end{algorithm}
\begin{algorithm}[ht]
\caption{InitialStateTest}
\label{alg:start_test}
\begin{algorithmic}[1]
\State \textbf{Input:} state cluster $i$, global episode counter $t$, number of initial state $N$, error bound $\epsilon$, failure probability $\delta$.
\If {t = 0}
\State Hypothetic state $s \leftarrow 0$
\State Initialize counter ${\tt cnt}(i) \leftarrow 0$ for $i \in [N]$
\State Initialize value estimates $v(i) \leftarrow 0$ for $i \in [N]$
\State $n_l \leftarrow \frac{8H^4 |\Acal|^3 \ln(N^2|\Fcal|/\delta)}{\epsilon (1-2\perturbradius) ^2}$
\State $n_t \leftarrow \frac{H^2\ln(N/\delta)}{2\epsilon^2}$
\State The algorithm instance will maintain the hypothesis state $s$, an episodes counter ${\tt cnt}(\cdot)$ and a value logger $v(\cdot)$ for the same state cluster $i$ across calls.
\EndIf 
\If {${\tt cnt}(s) < n_l$}
    \State Run $\algname$ for one episode on initial state $s$
    \State ${\tt cnt}(s) \leftarrow {\tt cnt}(s) + 1$
    \State \textbf{return} -1
\ElsIf  {$n_l \ge {\tt cnt}(s) < n_l + n_t$}
    \State Rollout learned policy $\policy$ for one episode and update value average $v(s)$
    \State ${\tt cnt}(s) \leftarrow {\tt cnt}(s) + 1$
    \State \textbf{return} -1
\Else
    \State $s \leftarrow s+1$
    \If {$s = N+1$}
    \State Run $\algname$ for one episode on initial state $\argmax_{s} v(s)$
    \State return $\argmax_{s} v(s)$
    \Else
    \State \textbf{return} -1
    \EndIf
\EndIf
\end{algorithmic}
\end{algorithm}

We present a variation of $\algname$ in~\pref{alg:stochastic_start} that can address stochastic initial state. The algorithm assumes access to an approximate decoder $\hat{\phi}: \Xcal \rightarrow \NN$ that can cluster observations from the same initial state together. This decoder can be thought of partitioning the observation space into \emph{decoder states} which recover the true initial states upto relabeling. In~\pref{sec:learning_decoder_initial_states} we discuss how to learn this decoder using the Homer algorithm~\cite{misra2020kinematic}. \pref{alg:stochastic_start} learns a mapping from these decoder states to initial states in $\Scal_1$ by performing a hypothesis testing algorithm that uses~\pref{alg:trembling_hand_decoding} in the main paper as a sub-routine. In~\pref{sec:aligning_decoder_states} we discuss this hypothesis test.

We will prove that under a realizability assumption (stated later), this algorithm has the following guarantee.
\begin{theorem}
\pref{alg:stochastic_start} will execute a policy that is close to robust policy by $\epsilon$ in all but
$poly\left\{N, H, A, \frac{1}{\epsilon}, \frac{1}{(1-2\eta)}, \frac{1}{\mumin}\ln\{\frac{1}{\delta}\}\right\}$
episodes with probability at least $1-\delta$.
\end{theorem}

Note that when there is a deterministic initial state, i.e., $N=1$ and $\mumin=1$, we recover dependence on the same set of parameters as our main result. For some problems, $N$ maybe significantly smaller than the set of all states. For example, a robot may start an episode from its charging station and there maybe a small number of charging stations in the environment. For these problems, the dependence on $N$ may be acceptable.

In the second part, we propose a hypothesis testing based algorithm, that use~\pref{alg:trembling_hand_decoding} in the main paper as a sub-routine, and prove the new sample complexity.

\subsection{Learning decoder initial states}
\label{sec:learning_decoder_initial_states}

We use the Homer algorithm~\citep{misra2020kinematic} to learn the decoder $\hat{\phi}$. We briefly describe the application of this algorithm for time step $h=1$. Homer collects a dataset $\Dcal$ of $n$ quads as follows: we sample $y$ uniformly in $\{0, 1\}$ and collect two independent transitions $(x^{(1)}, a^{(1)}, x'^{(1)}), (x^{(2)}, a^{(2)}, x'^{(2)})$ at the first time step by taking actions $a^{(1)}$ and $^{(2)}$ uniformly. If $y=1$ then we add $(x^{(1)}, a^{(1)}, x'^{(1)}, y)$ to $\Dcal$, otherwise, we add $(x^{(1)}, a^{(1)}, x'^{(2)}, y)$. Note that $(x^{(1)}, a^{(1)}, x'^{(2)})$ is an unobserved transition, therefore, we call it an imposter transition, whereas, $(x^{(1)}, a^{(1)}, x'^{(1)})$ is a real transition. We know that there are exactly $N$ initial states since we have access to the simulator. Given a bottleneck function class $\Phi: \{\phi: \Xcal \rightarrow [N]\}$ and another regressor class $\Gcal: \{f: [N] \times \Acal \times \Xcal \rightarrow [0, 1]\}$, we train a model to differentiate between real and imposter transition as follows:
\begin{equation*}
\hat{g}, \hat{\phi} = \arg\min_{g \in \Gcal, \phi \in \Phi} \sum_{(x, a, x', y) \in \Dcal} \left( g(\phi(x), a, x') - y \right)^2.   
\end{equation*}

\paragraph{Difference from~\cite{misra2020kinematic}.} While our approach and analysis in this subsection closely follows~\cite{misra2020kinematic}, we differ from them in two crucial ways. Firstly, we apply bottleneck on $x$ instead of $x'$, since we want to recover a decoder for initial states. Secondly,~\cite{misra2020kinematic} do not assume a margin assumption $\Gamma$ since their approach does not concern with recovering an exact decoder, but only in learning a good set of policies for exploration. 

We will denote the function class $\Gcal \circ \Phi = \{g\circ \phi: (x, a, x') \mapsto g(\phi(x), a, x') \mid g \in \Gcal, \phi \in \Phi\}$. Let $D(x, a, x')$ be the marginal distribution over real and imposter transitions, and let $\rho(x') = \Expt_{x \sim \mu, a \sim {\tt unf}(\Acal)}\left[ T(x' \mid x, a)\right]$ be the marginal probability over $x'$ for real transitions where $\mu$ is the initial state distribution. It can be shown that:
\begin{equation}
    D(x, a, x') = \frac{\mu(x)}{2|\Acal|}\left\{T(x' \mid x, a) + \rho(x')\right\},
\end{equation}
where $\nicefrac{T(x' \mid x, a) \mu(x)}{|\Acal|}$ is the probability of observing a real transition $(x, a, x')$ and $\nicefrac{\rho(x') \mu(x)}{|\Acal|}$ is the probability of observing an imposter transition $(x, a, x')$ and the factor of $\nicefrac{1}{2}$ comes due to uniform selection over real and imposter transition.

\cite{misra2020kinematic} showed that the Bayes optimal classifier of the prediction problem is given by:

\begin{lemma}[Bayes Optimal Classifier]\label{lem:homer-bayes-classifier} For any $(x, a, x')$ in support of $D$, we have:
\begin{equation*}
    g^\star(x, a, x') = \frac{T(x' \mid x, a)}{T(x' \mid x, a) + \rho(x')} = \frac{T(\phi^\star(x') \mid \phi^\star(x), a)}{T(\phi^\star(x') \mid \phi^\star(x), a) + \rho(\phi^\star(x'))} 
\end{equation*}
\end{lemma}
\begin{proof} See Lemma 9 of~\cite{misra2020kinematic}.
\end{proof}

Similar to~\cite{misra2020kinematic}, we make a realizability assumption stated below that allows us to solve the classification problem well.

\begin{assumption}[Realizability] We assume that $g^\star \in \Gcal \circ \Phi$.
\end{assumption}

We can use the realizability assumption to get the following generalization bound guarantee: for any $\delta \in (0, 1)$ we have:
\begin{equation}\label{eqn:generalization-bound-stochastic}
    \Expt_{x, a, x' \sim D}\left[ \left|\hat{g}(\hat{\phi}(x), a, x') -  g^\star(x, a, x')\right| \right] \le \Delta \defeq \sqrt{\frac{C(\Gcal \circ \Phi)}{n}\ln\left(\frac{1}{\delta}\right)},
\end{equation}
with probability at least $1-\delta$, where $C(\Gcal \circ \Phi)$ is a complexity measure for class $\Gcal \circ \Phi$ such as $\ln(|\Gcal ||\Phi|)$ or Rademacher complexity. For proof see Proposition 11 and Corollary 6 in~\cite{misra2020kinematic}. Note that even though their proof uses a bottleneck model $\phi$ on $x'$ instead of $x$, essentially the same argument holds by symmetry.

\paragraph{Coupling Distribution.} We define a coupling distribution as 
\begin{equation}
    D_c(x_1, x_2, a, x') = D(x_1 \mid a, x') D(x_2 \mid a, x') \frac{1}{|\Acal|}\rho(x'),
\end{equation} 
where $D(x_1 \mid a, x')$ is the conditional distribution derived from the joint distribution $D(x_1, a, x')$ defined earlier. We also define the marginal distribution $D(a, x')$ which gives us $D(x_1 \mid a, x') = \nicefrac{D(x_1, a, x')}{D(a, x')}$. 

We present some result related to the distributions defined above that will be useful later for proving important results later.
\begin{align}
    D(a, x') &= \sum_{x}D(x, a, x') = \sum_{x} \frac{\mu(x)}{2|\Acal|} \{T(x' \mid x, a) + \rho(x')\} \ge \frac{\rho(x')}{2|\Acal|},\label{eqn:lower-bound-dax}\\
    \sum_{a} D(a, x') &= \rho(x').
\end{align}

Using~\pref{eqn:lower-bound-dax} we can prove:
\begin{align}
    \sum_{x_1} D_{c}(x_1, x_2, a, x') &= \sum_{x_1} D(x_1\mid a, x') \frac{D(x_2, a, x')}{D(a, x')} \frac{\rho(x')}{|\Acal|} \\
    &\le 2\sum_{x_1} D(x_1\mid a, x') D(x_2, a, x') = 2D(x_2, a, x').
\end{align}
Similarly, we can prove:
\begin{equation}\label{eqn:coupling-to-d}
\sum_{x_2} D_c(x_1, x_2, a, x') \le 2D(x_1, a, x').
\end{equation}

Further, we have:
\begin{align*}\label{eqn:conditional}
    D(x \mid a, x') &=  \frac{\mu(x)\{T(x' \mid x, a) + \rho(x')\}}{2|\Acal|D(a, x')} 
    \ge  \frac{\mu(x)\rho(x')}{2|\Acal|D(a, x')} \ge  \frac{\mu(x)\rho(x')}{2|\Acal|\sum_{a \in \Acal} D(a, x')}= \frac{\mu(x)}{2|\Acal|},
\end{align*}
which gives us:
\begin{equation}
    \frac{D(x \mid a, x')\rho(x')}{T(x' \mid x, a) + \rho(x')} = \frac{\mu(x)\rho(x')}{2|\Acal| D(a, x')} \ge \frac{\mu(x)}{2|\Acal|}.\label{eqn:ratio}
\end{equation}

We now state a useful lemma.
\begin{lemma}\label{lem:coupling-lem} Fix $\delta \in (0, 1)$. Then with probability at least $1-\delta$ we have 
\begin{equation*}
    \Expt_{x_1, x_2, a, x' \sim D_{c}}\left[\one\{\hat{\phi}(x_1) = \hat{\phi}(x_2)\}|g^\star(x_1, a, x') - g^\star(x_2, a, x')|\right] \le 4\Delta.
\end{equation*}
\end{lemma}
\begin{proof} We use triangle inequality to decompose the left hand side as:
\begin{align*}
    & \Expt_{(x_1, x_2, a, x') \sim D_{c}}\left[\one\{\hat{\phi}(x_1) = \hat{\phi}(x_2)\}|g^\star(x_1, a, x') - g^\star(x_2, a, x')|\right]  \\
    &\le \Expt_{(x_1, x_2, a, x') \sim D_{c}}\left[\one\{\hat{\phi}(x_1) = \hat{\phi}(x_2)\}|g^\star(x_1, a, x') - \hat{g}(\hat{\phi}(x_1), a, x')|\right]  + \\
    & \Expt_{(x_1, x_2, a, x') \sim D_{c}}\left[\one\{\hat{\phi}(x_1) = \hat{\phi}(x_2)\}|\hat{g}(\hat{\phi}(x_1), a, x') - g^\star(x_2, a, x')|\right]
\end{align*}

The first term is bounded as shown below:
\begin{align*}
    & \Expt_{(x_1, x_2, a, x') \sim D_{c}}\left[\one\{\hat{\phi}(x_1) = \hat{\phi}(x_2)\}|g^\star(x_1, a, x') - \hat{g}(\hat{\phi}(x_1), a, x')|\right] \\
    &\le \Expt_{(x_1, x_2, a, x') \sim D_{c}}\left[|g^\star(x_1, a, x') - \hat{g}(\hat{\phi}(x_1), a, x')|\right]\\
     &\le 2\Expt_{(x, a, x') \sim D}\left[|g^\star(x, a, x') - \hat{g}(\hat{\phi}(x), a, x')|\right] = 2\Delta,
\end{align*}
where the second inequality uses~\pref{eqn:coupling-to-d} and~\pref{eqn:generalization-bound-stochastic}. The second term is bounded as:
\begin{align*}
    & \Expt_{(x_1, x_2, a, x') \sim D_{c}}\left[\one\{\hat{\phi}(x_1) = \hat{\phi}(x_2)\}|\hat{g}(\hat{\phi}(x_1), a, x') - g^\star(x_2, a, x')|\right]  \\
    & = \Expt_{(x_1, x_2, a, x') \sim D_{c}}\left[\one\{\hat{\phi}(x_1) 
    = \hat{\phi}(x_2)\}|\hat{g}(\hat{\phi}(x_2), a, x') - g^\star(x_2, a, x')|\right]  \le 2\Delta,
\end{align*}
where the inequality results from following similar steps used for bounding the first term. Adding the two upper bounds we get $4\Delta$.
\end{proof}

Using~\pref{lem:homer-bayes-classifier}, we have for every $x_1, x_2, x' \in \Xcal, a \in \Acal$:
\begin{equation}
|g^\star(x_1, a, x') - g^\star(x_2, a, x')| = \frac{\rho(x')|T(x' \mid x_1, a) - T(x' \mid x_2, a)|}{(T(x' \mid x_1, a) + \rho(x')) (T(x' \mid x_2, a) + \rho(x'))}.
\end{equation}

We use this to prove the following result:
\begin{lemma}\label{lem:one-sided-error} With probability at least $1-\delta$ we have:
\begin{equation*}
    \Pr_{x_1, x_2 \sim \mu}\left(\phi^\star(x_1) \ne \phi^\star(x_2) \land \hat{\phi}(x_1) = \hat{\phi}(x_2) \right) \le \frac{8 |\Acal^2| \Delta}{\Gamma}.
\end{equation*}
\end{lemma}
\begin{proof} Let's define a shorthand $\Ecal = \one\{\hat{\phi}(x_1) = \hat{\phi}(x_2)\}$. Starting with left hand side of~\pref{lem:coupling-lem} we get:
\begin{align*}
    & \Expt_{(x_1, x_2, a, x') \sim D_c}\left[\Ecal |g^\star(x_1, a, x') - g^\star(x_2, a, x')|\right] \\
    &= \Expt_{(x_1, x_2, a, x') \sim D_c}\left[\Ecal\frac{\rho(x')|T(x' \mid x_1, a) - T(x' \mid x_2, a)|}{(T(x' \mid x_1, a) + \rho(x')) (T(x' \mid x_2, a) + \rho(x'))}\right]\\
    &= \sum_{x_1, x_2, a, x'} \frac{\Ecal}{|\Acal|} \frac{D(x_1 \mid x', a)\rho(x')}{(T(x' \mid x_1, a) + \rho(x')) } \frac{ D(x_2 \mid a, x')\rho(x')}{(T(x' \mid x_2, a) + \rho(x'))} |T(x' \mid x_1, a) - T(x' \mid x_2, a)|\\
    &\ge \sum_{x_1, x_2, a, x'} \Ecal \frac{\mu(x_1)\mu(x_2)}{4|\Acal|^3} |T(x' \mid x_1, a) - T(x' \mid x_2, a)|, \quad \mbox{using Equation~\ref{eqn:ratio}}\\
    &\ge \sum_{x_1, x_2, a} \one\{\phi^\star(x_1) \ne \phi^\star(x_2)\} \Ecal
    \frac{\mu(x_1)\mu(x_2)}{2|\Acal|^3} \Gamma \\
    &= \frac{\Gamma}{2|\Acal|^2}\Pr_{x_1, x_2 \sim \mu}\left(\phi^\star(x_1) \ne \phi^\star(x_2) \land \hat{\phi}(x_1) = \hat{\phi}(x_2) \right)
\end{align*}
where the last inequality uses $\frac{1}{2}\sum_{x'}|T(x' \mid x_1, a) - T(x' \mid x_2, a)| = \|T(\cdot \mid x_1, a) - T(\cdot \mid x_2, a)\|_{{\tt TV}}$ which is either zero when $\phi^\star(x_1) = \phi^\star(x_2)$ or at least $\Gamma$. Combining these two conditions we get a lower bound of $\one\{\phi^\star(x_1) \ne \phi^\star(x_2)\} \Gamma$ on TV distance. The proof is then completed with application of~\pref{lem:coupling-lem}.
\end{proof}

\begin{theorem}(Initial State Clustering Result).
\label{thm:initial_state_clustering}  Let $N>1$ and let $\Delta < \frac{\mumin^2\Gamma}{32N^2|\Acal|^2}\left(1  - \left(1 - \frac{2\mumin}{N}\right)^2 \right)$. Then there exists a bijection mapping $\sigma: \Scal_1 \rightarrow [N]$ such that with probability at least $1-\delta$:
\begin{equation}
    \forall s \in \Scal_1,~~\Pr_{x \sim \mu}\left(\phi^\star(x) = s \mid \hat{\phi}(x) = \sigma(s)\right) > 1 - \frac{16N^2 |\Acal|^2\Delta}{\Gamma \mumin^2} = 1-\frac{16N^2 |\Acal|^2}{\Gamma\mumin^2} \sqrt{\frac{C(\Fcal \circ \Phi)\ln(1/\delta)}{n}}.
\end{equation}
\end{theorem}

\begin{proof} We will use $i \in [N]$ to denote a decoder state defined by $\{x \mid \hat{\phi}(x) = i, x \in \Xcal_1\}$ and $s \in \Scal_1$ to denote a real state defined by $\{x \mid \phi^\star(x) = s, x \in \Xcal_1\}$. For any $i, s$ we can bound the left hand side of~\pref{lem:one-sided-error} as:
\begin{align}
    & \Pr(\phi^\star(x_1) \ne \phi^\star(x_2) \land \hat{\phi}(x_1) = \hat{\phi}(x_2)) \\
    &=  \Pr(\cup_{j \in [N], \tilde{s} \in \Scal_1} \phi^\star(x_1) = \tilde{s}, \phi^\star(x_2) \ne \tilde{s}, \hat{\phi}(x_1) = j,  \hat{\phi}(x_2) = j)\\
    &\ge  \Pr(\phi^\star(x_1) = s, \phi^\star(x_2) \ne s, \hat{\phi}(x_1) = i,  \hat{\phi}(x_2) = i)\\
    &=  \Pr(\phi^\star(x_1) = s, \hat{\phi}(x_1) = i) \Pr(\phi^\star(x_2) \ne s, \hat{\phi}(x_2) = i) \\
    &=  \Pr(\phi^\star(x) = s, \hat{\phi}(x) = i) \left\{\Pr(\hat{\phi}(x) = i) - \Pr(\phi^\star(x) = s, \hat{\phi}(x) = i)\right\}
\end{align}
where the second last step follows from observing that $x_1$ and $x_2$ are sampled independently. We define a few shorthands: $\Pr(i) = \Pr(\hat{\phi}(x) = i)$, $\Pr(s) = \Pr(\phi^\star(x) = s) = \mu(s)$ and $\Pr(i, s) = \Pr(\phi^\star(x) = s, \hat{\phi}(x) = i)$. This combined with above and~\pref{lem:one-sided-error} gives us:
\begin{equation}
    \forall i \in [N], s \in \Scal_1, \qquad \Pr(i, s) \left(\Pr(i) - \Pr(i, s)\right) \le \Delta' \defeq\frac{8|\Acal|^2\Delta}{\Gamma} 
\end{equation}
We define a mapping $\sigma: \Scal_1 \rightarrow [N]$ as follows:
\begin{equation}
    \sigma(s) = \arg\max_{i \in [N]} \Pr(i, s)
\end{equation}
This gives us:
\begin{equation}\label{eqn:one-sided-error-simplified}
    \Pr(\sigma(s), s)\left(\Pr(\sigma(s)) - \Pr(\sigma(s), s) \right) \le \Delta'
\end{equation}
Since $\Delta'$ can be brought arbitrarily small, we will assume $\Delta' < \nicefrac{\Pr(\sigma(s))^2}{4}$ which allows us to write:
\begin{align}
    \Pr(\sigma(s), s) &> \frac{\Pr(\sigma(s)) + \sqrt{\Pr(\sigma(s))^2 - 4\Delta'}}{2}, \mbox{ or} \label{eqn:upper-bound-pr-sigma} \\
     \Pr(\sigma(s), s) &< \frac{\Pr(\sigma(s)) - \sqrt{\Pr(\sigma(s))^2 - 4\Delta'}}{2}\label{eqn:lower-bound-pr-sigma}
\end{align}
By definition of $\sigma(s)$ we have:
\begin{equation}
    \Pr(\sigma(s), s) \ge \frac{1}{N}\sum_{i=1}^N \Pr(i, s) = \frac{\Pr(s)}{N} \ge \frac{\mumin}{N},\label{eqn:lower-bound-sigma}
\end{equation}
where the first inequality uses the fact that maximum of a set of values is greater than its average, and the last inequality uses~\pref{assum:minimum_prob}. We now place another condition on $\Delta'$, namely, 
\begin{equation}\label{eqn:delta-prime-second-constraint}
    \Delta' < \frac{\Pr(\sigma(s))^2}{4}\left(1  - \left(1 - \frac{2\mumin}{\Pr(\sigma(s))N}\right)^2 \right),
\end{equation}
which implies that:
\begin{equation}
    \Pr(\sigma(s), s) < \frac{\Pr(\sigma(s)) - \sqrt{\Pr(\sigma(s))^2 - 4\Delta'}}{2} < \frac{\mumin}{N}.
\end{equation}
This eliminates~\pref{eqn:lower-bound-pr-sigma}. Hence the $\Pr(\sigma(s), s)$ must satisfy~\pref{eqn:upper-bound-pr-sigma} which can be simplified as:
\begin{align}
    \Pr(\sigma(s), s) &> \frac{\Pr(\sigma(s)) + \sqrt{\Pr(\sigma(s))^2 - 4\Delta'}}{2} \\
    &= \frac{\Pr(\sigma(s))}{2}\left(1 + \left(1 - \frac{4\Delta'}{\Pr(\sigma(s))^2}\right)^{\frac{1}{2}} \right)\\
    &\ge \Pr(\sigma(s))\left(1 - \frac{2\Delta'}{\Pr(\sigma(s))^2} \right)\\
    &\ge \Pr(\sigma(s))\left(1 - \frac{2N^2\Delta'}{\mumin^2} \right)
\end{align}
where the third step uses $\sqrt{1-y} \ge 1 -y$ for $y \in [0, 1]$ and that $\nicefrac{4\Delta'}{\Pr(\sigma(s))^2} < 1$ from constraints on $\Delta'$, and the last step uses $\Pr(\sigma(s)) \ge \Pr(\sigma(s), s) \ge \nicefrac{\mumin}{N}$ (\pref{eqn:lower-bound-sigma}). We can finally prove our main result as:
\begin{equation*}
    \Pr(s \mid \sigma(s)) = \frac{\Pr(s, \sigma(s))}{\Pr(\sigma(s))} \ge 1 - \frac{2N^2\Delta'}{\mumin^2} = 1 - \frac{16N^2 |\Acal|^2\Delta}{\Gamma \mumin^2}.
\end{equation*}
What is left is to show that $\sigma(s)$ is a bijection mapping and collect all constraints on $\Delta'$. Let $s$ and $s'$ be two initial states such that $\sigma(s) = \sigma(s') = k$. We then get:
\begin{equation}
    \Pr(k, s)  \Pr(k, s') \le \Pr(k, s) \left(\Pr(k) - \Pr(k, s) \right) \le \Delta'
\end{equation}
where the last equality follows from~\pref{eqn:one-sided-error-simplified}. Further, we have $ \Pr(k, s) \ge \nicefrac{\mumin}{N}$  and $\Pr(k, s') \ge \nicefrac{\mumin}{N}$ from~\pref{eqn:lower-bound-sigma}. This gives us $\frac{\mumin^2}{N^2} \le\Delta'$. Hence, if $\Delta' < \frac{\mumin^2}{N^2}$, then, we cannot have two different initial states mapping to the same decoder state. Further, as $|\Scal_1| = N$, hence the map $\sigma: \Scal_1 \rightarrow [N]$ is a bijection.

Finally, we made three constraints on $\Delta'$. The first is $\Delta' < \nicefrac{\Pr(\sigma(s))^2}{4}$, second is~\pref{eqn:delta-prime-second-constraint}, and third is $\Delta' < \frac{\mumin^2}{N^2}$. Note that~\pref{eqn:delta-prime-second-constraint} already implies that $\Delta' < \nicefrac{\Pr(\sigma(s))^2}{4}$. Hence, the overall constraint on $\Delta'$ is:
\begin{equation}
    \Delta' < \min\left\{\frac{\mumin^2}{N^2}, \frac{\Pr(\sigma(s))^2}{4}\left(1  - \left(1 - \frac{2\mumin}{\Pr(\sigma(s))N}\right)^2 \right)\right\}
\end{equation}
We can simplify this constraint by making it tighter using $\Pr(\sigma(s)) \in \left[\frac{\mumin}{N}, 1\right]$:
\begin{equation}
    \Delta' < \min\left\{\frac{\mumin^2}{N^2}, \frac{\mumin^2}{4N^2}\left(1  - \left(1 - \frac{2\mumin}{N}\right)^2 \right)\right\} = \frac{\mumin^2}{4N^2}\left(1  - \left(1 - \frac{2\mumin}{N}\right)^2 \right).
\end{equation}
Note that we are assuming here that $N \ge 2$ and, therefore, $\mumin < 1$, which implies $\nicefrac{2\mumin}{N} \le \mumin < 1$. When $N=1$, we can trivially align the single initial state. This completes the proof. 
\end{proof}

This allows us to separate all initial states at time step $h=1$ with high probability.

\subsection{Aligning the learned decoder states}
\label{sec:aligning_decoder_states}

The only thing left is aligning learned decoder states in the target domain with simulator initial states, which we do as follows.

At a high level Algorithm \ref{alg:stochastic_start} first clusters the initial observation into clusters, then each time it start with a cluster, it run a subroutine, InitialStateTest, to test the latent state index of that cluster. The testing algorithm Algorithm \ref{alg:start_test}, run our main algorithm $\algname$ the hypothesis of the state index from $0$ to $N$. By the analysis of $\algname$, we know that if the hypothesis is correct, it will learn a policy with nearly robust value. Thus for each cluster, with in $Nn$ episodes the algorithm will find the nearly robust policy.

\begin{theorem}
If we learn the initial state decoder $\hat{\phi}$ with $n_0$ samples such that 
$$ n_0 \ge \max\left\{  \frac{256H^2N^4 |\Acal|^4 C(\Fcal \circ \Phi)\ln(1/\delta)}{ \epsilon^2 \Gamma^2 \eta^4}, \frac{1024N^4 |\Acal|^4C(\Fcal \circ \Phi)\ln(1/\delta)}{\Gamma^2 \eta^6} \right\},$$ 
Algorithm \ref{alg:stochastic_start} will execute a policy that is close to robust policy by $4\epsilon$ in all but
$$ \Ocal\left(\frac{N^2 H^4 |\Acal|^3\ln(N^2|\Fcal|/\delta)}{\epsilon(1-2\perturbradius) ^2}+ \frac{N^2H^2\ln(N^2/\delta)}{\epsilon^2}  + \frac{N^4 |\Acal|^4C(\Fcal \circ \Phi)\ln(1/\delta)}{\Gamma^2\eta^2} \max\left\{ \frac{H}{\epsilon^2}, \frac{1}{\eta^2} \right\} \right)$$
episodes with probability at least $1-3\delta$.
\end{theorem}
\begin{proof}
We prove this theorem by two steps. First, we show that for each cluster $\hat{\phi}(x)$, after we run Algorithm \ref{alg:start_test} with $N(n_l+n_t)$ steps we can find the correct state of that cluster. Second, we show that once we find the correct state, we run $\algname$ and learn a policy at most $4\epsilon$ worse than the robust policy.

First, by the definition of $n_0$, we have that $\Delta := \sqrt{\frac{C(\Fcal \circ \Phi)}{n_0}\ln\left(\frac{1}{\delta}\right)} \le \frac{\eta^2\Gamma}{32N^2|\Acal|^2} \frac{2\eta}{N} < \frac{\eta^2\Gamma}{32N^2|\Acal|^2}\left(1  - \left(1 - \frac{2\eta}{N}\right)^2 \right)$ for $N \ge 2$ and $\eta < 1$. By Theorem \ref{thm:initial_state_clustering}, we have for each $i \in [N]$, there must exist a state $\sigma(i) \in [N]$ such that with probability at least $1-\delta$,
\begin{align}
    \Pr(\phi(x) = s | \hat{\phi}(x) = i) &\ge 1 - \frac{16N^2 |\Acal|^2\Delta}{\Gamma \eta^2} \\
    &\ge 1-\frac{\epsilon}{H}
\end{align}
Thus if we run Algorithm \ref{alg:start_test} InitialStateTest for $s = \sigma(i)$. Then by Theorem \ref{thm:robust_tranfer_stronger}, we have that the value of learned policy $\policy$ is at least 
$$
V^{\rho \circ \inverseobsmap^\star}_{\blockM} - \frac{8H^4 |\Acal|^3 \ln(N^2|\Fcal|/\delta)}{n_l(1-2\perturbradius) ^2} - H \epsilon_0 = V^{\rho \circ \inverseobsmap^\star}_{\blockM} - 2\epsilon,
$$ with probability at least $1-\delta/N^2$.
By Hoeffding's inequality, we know that if the Monte-Carlo estimates $v(s)$ of the learned policy value with $N_t$ samples is at least 
$$
V^{\rho \circ \inverseobsmap^\star}_{\blockM} - 2\epsilon - \sqrt{\frac{H^2\ln(N^2/\delta)}{2n_t}} = V^{\rho \circ \inverseobsmap^\star}_{\blockM} - 3\epsilon,
$$
with probability at least $1-\delta/N^2$. Thus let $\hat{s} = \argmax_s v(s)$ and $\pi$ be the corresponding learned policy from $\algname$ given hypothetical initial state $\hat{s}$. The policy value $v^{\policy}$ is at least 
\begin{align}
    v(\hat{s}) - \epsilon \ge V^{\rho \circ \inverseobsmap^\star}_{\blockM} - 4\epsilon
\end{align}
with probability $1-2\delta/N$ by taking the union bound on all $N$ possible hypothetical initial states.  Thus, for a given state cluster, after run Algorithm \ref{alg:start_test} InitialStateTest for $N(n_l+n_t)$ episodes, the policy value is at least $V^{\rho \circ \inverseobsmap^\star}_{\blockM} - 4\epsilon$. That means for each state cluster, we make at most $N(n_l+n_t)$ mistakes. Thus we make at most $N^2(n_l+n_t)$ mistakes in total during running Algorithm \ref{alg:start_test}. Notice that we also need $n_0$ samples to learn the decoder $\hat{\phi}$. The total number of episodes we may make mistake on is therefore 
$$ N^2(n_l+n_t) + n_0 $$
Taking the union bound over all initial state clusters $\hat{\phi}(x)$ and the high probability statement in Theorem \ref{thm:initial_state_clustering}, we have that with probability $1-3\delta$ the statement holds. We finish the proof by plugging in $n_l, n_t, n_0$.
\end{proof}

\section{Experiment Details}
\label{appendix:experiment}
\subsection{Experiment Details in Combination Lock}

\paragraph{Domain details}
We describe the details of the combination lock domain here. The deterministic MDP is described in section \ref{sec:experiments}. The transition dynamics and reward functions in the target domain is defined by the the coefficients $\perturbprob(a'|s,a)$ and Definition \ref{def:perturb_mdp_class}. For each $s,a$, $\perturbprob(a'|s,a) := \perturbradius p(a') + (1-\perturbradius) \ind (a'=a) $ where $p(\cdot)$ is a random probability mass distribution and each probability is draw uniformly from $[0,1]$ then normalized. $\eta$ is set to be $0.1$ in this experiment. The only exception is that all transitions from state $H,2$ are not perturbed. This settings makes the robust policy without this knowledge not optimal in this specific instance, but our goal is still learning the robust policy here.

The observation mapping is defined as below. Let $v_s$ be a one-hot embedding in $\mathbb{R}^{H+4}$ of the state where the first $3$ bits denotes the second number in state representation and the later $H+1$ bits denotes the first number (time-step index) in the state representation. The the observation $o$ is computed by
\begin{align}
    o = \mathbf{H} \times {\tt perm}(v_s + \mathcal{N}(0,0.01) ) ,
\end{align}
where ${\tt perm}$ is an arbitrary permutation of $[H+4]$, and $\mathbf{H}$ is the $2^{\lceil \log_2(H+4) \rceil }$ by $2^{\lceil \log_2(H+4) \rceil }$ Hadamard matrix, consists of $2^{\lceil \log_2(H+4) \rceil }$ mutually orthogonal columns. If $H+4 < 2^{\lceil \log_2(H+4) \rceil }$, then the vector ${\tt perm}(v_s + \mathcal{N}(0,0.01) )$ will be padded with zeros.

\paragraph{Implementation details of $\algname$}

We describe the details of $\algname$ and hyper-parameters below. We implement the action predictor by $H$ two-layer MLPs with ReLU activations in this domain. The input to the MLP is concatenated from the two observations $x$ and $x'$. We train them in the forward order each with $n/H$ episides and $n$ is the total number of episodes. Since our algorithm is not online and needs a predefined sample size $n$, we uses $n \in \{ 10000, 50000, 100000, 400000 \} $ for our algorithms. The reported results for each combination lock problem (with different horizons) is the smallest $n$ such that the median of results on 5 random seeds reaches the 0.95 of robust value.

The $h+1$-th action predictor is initialized with the parameter from $h$-th action predictor, and the first one is using PyTorch's default initialization. For action predictor we split the training and a hold-out validation set by $0.8:0.2$. The trained model is tested on validation set after each epoch, and we stop the training if there is no decreasing in the negative log-likelihood loss for $10$ epoches, or is trained more than $100$ epoches. Other hyper-parameters are specified in Table \ref{tab:hp_nn_combolock}.

\begin{table}[ht]
    \centering
    \begin{tabular}{c|c}
        \toprule
         parameter name & values  \\
         \midrule
         hidden layer sizes & $[56, 56]$ \\
         learning rate & $0.0003$ \\
         optimizer & Adam \\
         batch size & $32$ \\
         gradient clipping & 0.25 \\
         \bottomrule 
    \end{tabular}
    \vspace{0.1cm}
    \caption{Hyperparameters in training NNs for all algorithms in the combination lock domain}
    \label{tab:hp_nn_combolock}
\end{table}

\paragraph{Implementation details of PPO, PPO+RND, and domain randomization}
We train PPO and PPO+RND, with a maximum of $5 \times 10^5$ number of episodes in the target domain. The policy network is implemented by a two-layer MLP with ReLU activations in this domain, and the training hyperparameters are the same as Table \ref{tab:hp_nn_combolock}. When training RND, the observation is normalized to mean $0$ and standard deviation $1$. Other hyperparameters that are specific to the PPO and RND is listed on Table \ref{tab:hp_ppo_combolock}.

For the domain randomization pretraining, we train in the source domain with $5 \times 10^5$ episodes first. Every 100 episodes the domain is randomized uniformly by regenerating a permutation function and the perturbation in transitions $\perturbprob(\cdot|\cdot,\cdot)$, in the same way as how target domain is generated. We assume the domain randomization algorithm know $\perturbradius$ (same as $\algname$), and everything about the observation mapping except the specific permutation function (not known by $\algname$).

Rather than uniform domain randomization, EPOpt \citep{rajeswaran2016epopt} uses only the episodes with reward smaller than lower $\epsilon$-quantile in the each batch ($100$ episodes) during domain randomization pretraining. We searched over the $\epsilon$ and it turns out uniform domain randomization ($\epsilon = 1.0$) performs the best in this problem.

\begin{table}[ht]
    \centering
    \begin{tabular}{c|c}
        \toprule
         parameter name & values  \\
         \midrule
         clip epsilon & 0.1 \\
         discounting factor & $0.999$ \\
         number of updates per batch & $4$ \\
         number of episodes per batch & $100$ \\
         entropy loss coefficient & $0.001$ \\
         RND bonus coefficient & 500 (0, 100, 500) \\
         $\epsilon$ in EPOpt & 1.0 (0.1, 0.2, 0.5, 1.0) \\
         \bottomrule 
    \end{tabular}
    \vspace{0.1cm}
    \caption{Hyperparameters for PPO, PPO+RND, PPO(+RND) with domain randomization in the combination lock domain. Values in the parenthesis is the values that we searched, with the picked best value outside. }
    \label{tab:hp_ppo_combolock}
\end{table}

\paragraph{Details of the experiment results}
We report the details of experiments we used to generate Figure \ref{fig:combolock_diff_h} in the paper. We run each algorithm with 5 random seeds for a maximum of $5 \times 10^5$ number of episodes in the target domain, saving checkpoints every $1000$ episodes. We decide whether the combination lock is solved or not by a threshold of $0.95 v^\robustlatentpolicy$, where $\robustlatentpolicy$ is the latent robust policy with a latent state access in the target domain. In Figure \ref{fig:combolock_diff_h} we report the median of the number of episodes needed to solve the combination lock with horizon $H$. The results of each random seeds is list in the following table.

\begin{table}[ht]
    \centering
    \begin{tabular}{c|c|c}
    \toprule
       Algorithm & Horizon & Number of episodes needed  \\
       \midrule
       $\algname$  & 5 & 10000, 9000, 12000, 10000, 11000 \\
       & 10 & 10000, 11000, $\infty$, 11000, $\infty$ \footnotemark \\
       & 20 & 50000, 48000, 48000, 48000, 49000 \\
       & 40 & 393000, 391000, 391000, 393000, 395000 \\
       \midrule
       PPO & 5 & 8000, 5000, 6000, 7000, 5000 \\
       & $\ge 10$ & $\infty$, $\infty$, $\infty$, $\infty$, $\infty$ \\
       \midrule
       PPO+RND & 5 & 7000, 7000, 8000, 6000, 5000 \\
       &10 & 14000, 47000, 10000, 16000, 14000 \\
       &20 & 121000, 120000, 79000, 79000, 90000 \\
       & $\ge 40$ & $\infty$, $\infty$, $\infty$, $\infty$, $\infty$ \\
       \midrule
       PPO+DR & 5 & 2000, 6000, 6000, 7000, 1000 \\
       & $> 5$ & $\infty$, $\infty$, $\infty$, $\infty$, $\infty$ \\
       \midrule
       PPO+RND+DR & 5 & 6000, 2000, 1000, 7000, 3000 \\
       & 10 & 62000, 90000, $\infty$, $\infty$, 88000 \\
       & $\ge 20$ & $\infty$, $\infty$, $\infty$, $\infty$, $\infty$ \\
       \bottomrule
    \end{tabular}
    \vspace{0.2cm}
    \caption{Results for all random seeds to generate Figure \ref{fig:combolock_diff_h}}
    \label{tab:combo_lock_full_results}
\end{table}
\footnotetext[1]{The $\infty$ means with the pre-defined samples size $10000$ we did not reaches the 0.95 of robust value in that random seeds. However it does not means our algorithm needs to take infinite amount of episodes to solve it. With the next sample size value $50000$ our algorithm can solve it. This is verified by the results on a harder problem H=20. The reported results for each problem horizon is the smallest $n$ in $\{ 10000, 50000, 100000, 400000 \}$ such that the median of results on 5 random seeds solves the problem (reaching the 0.95 of robust value). }

In the main paper, Figure \ref{fig:reward_curve_h40_combolock} shows the reward curves for different algorithms with horizon equals $40$ in the target environment. Here we include the reward curves for $H$ is $5$, $10$, and $20$ in Figure \ref{fig:combo_lock_rewards}. It shows our algorithm stably learned a robust policy with number of samples that is smaller than baseline algorithms. Though baseline algorithms aim to learn the near optimal policy, they did not converge to a policy that is significantly better than the robust policy. 

\begin{figure*}[!t]
    \centering
    \begin{minipage}{.31\textwidth}
        \centering
        \includegraphics[width=1\textwidth]{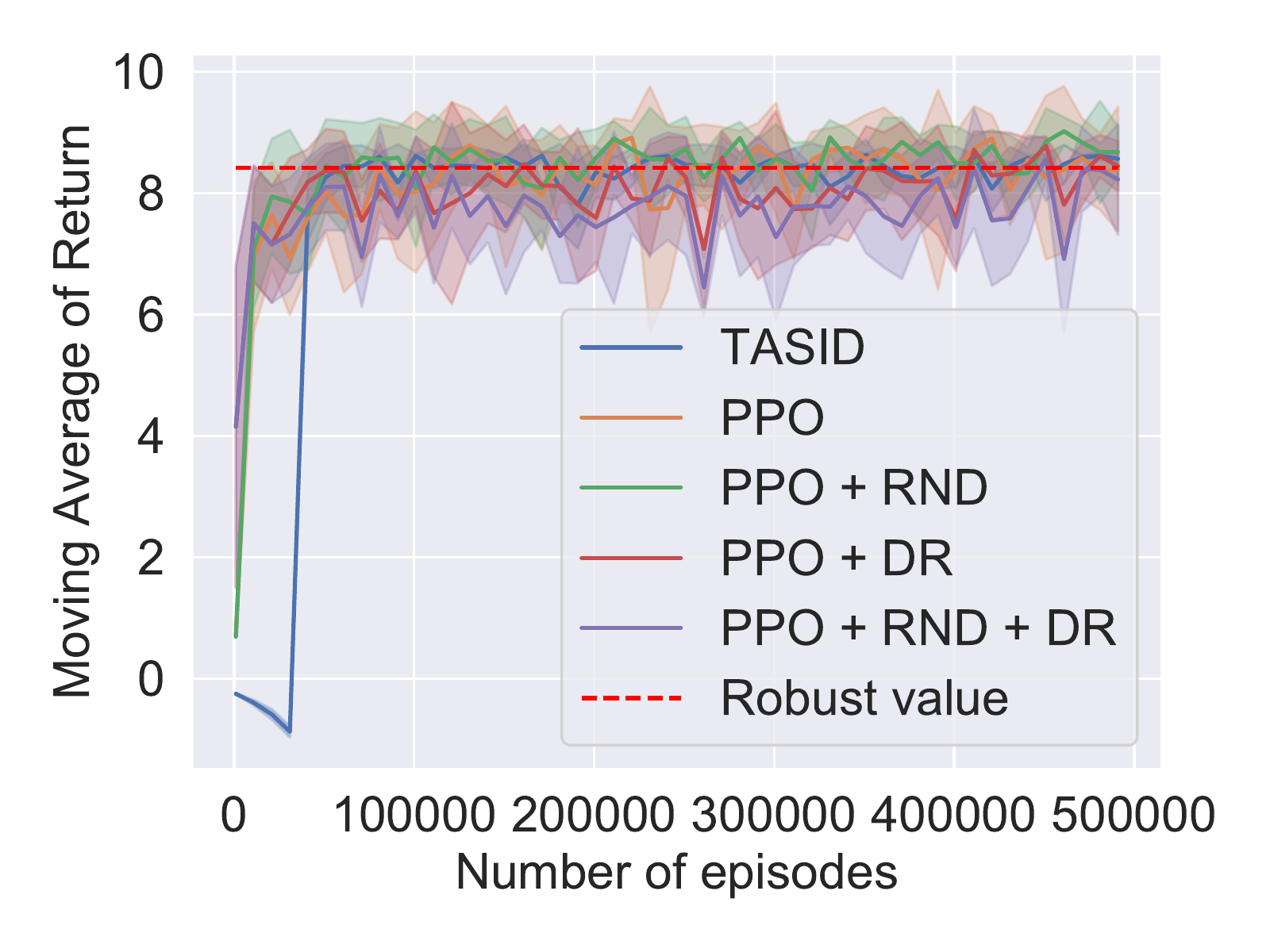}
        \subcaption{$H=5$}
        \label{fig:reward_curve_h5_combolock}
    \end{minipage}
    \hspace{0.05cm}
    \begin{minipage}{.31\textwidth}
        \centering
       \includegraphics[width=1\textwidth]{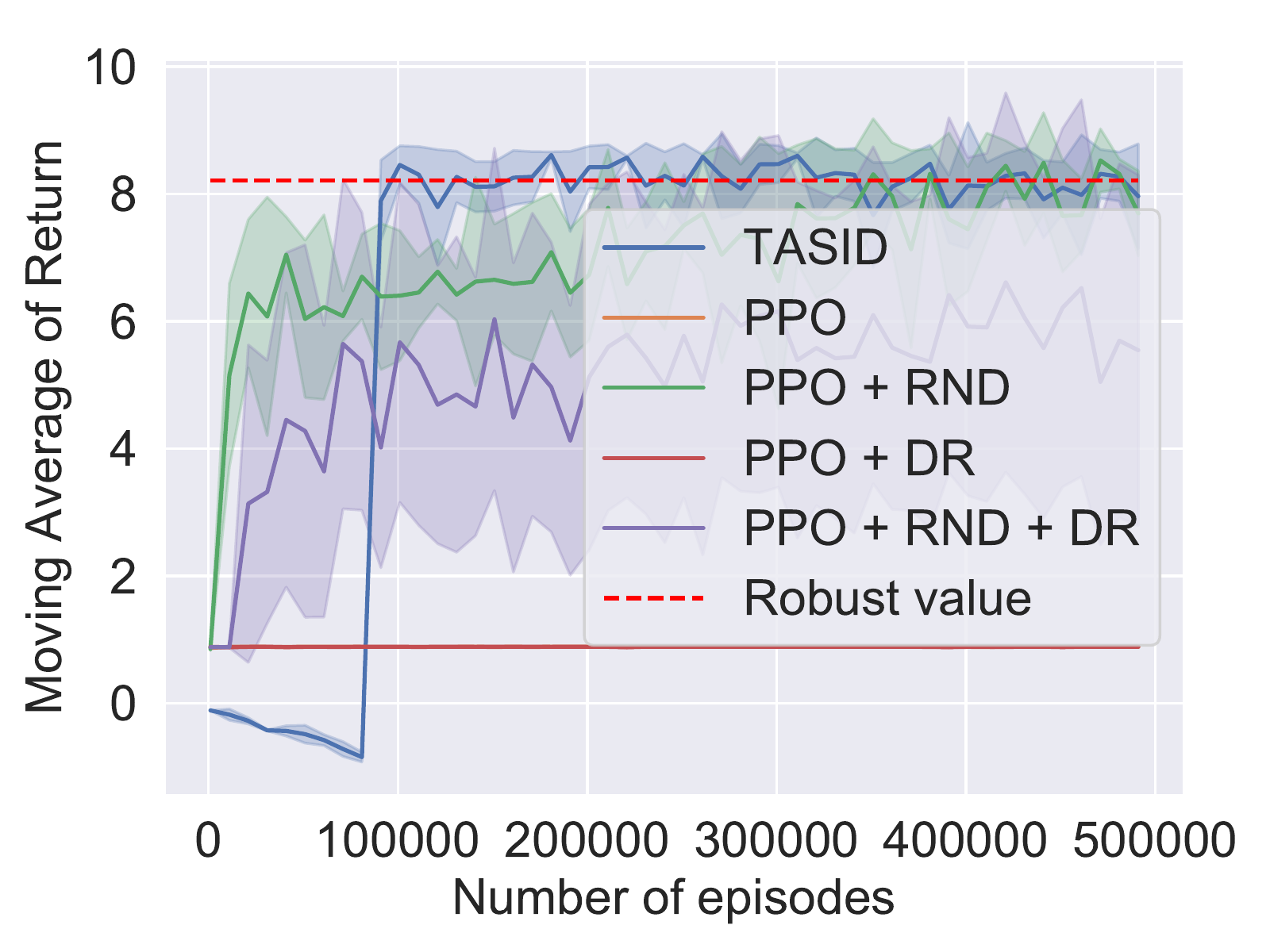}
        \subcaption{$H=10$}
        \label{fig:reward_curve_h10_combolock}
    \end{minipage}
    \hspace{0.05cm}
    \begin{minipage}{.31\textwidth}
        \centering
        \includegraphics[width=1\textwidth]{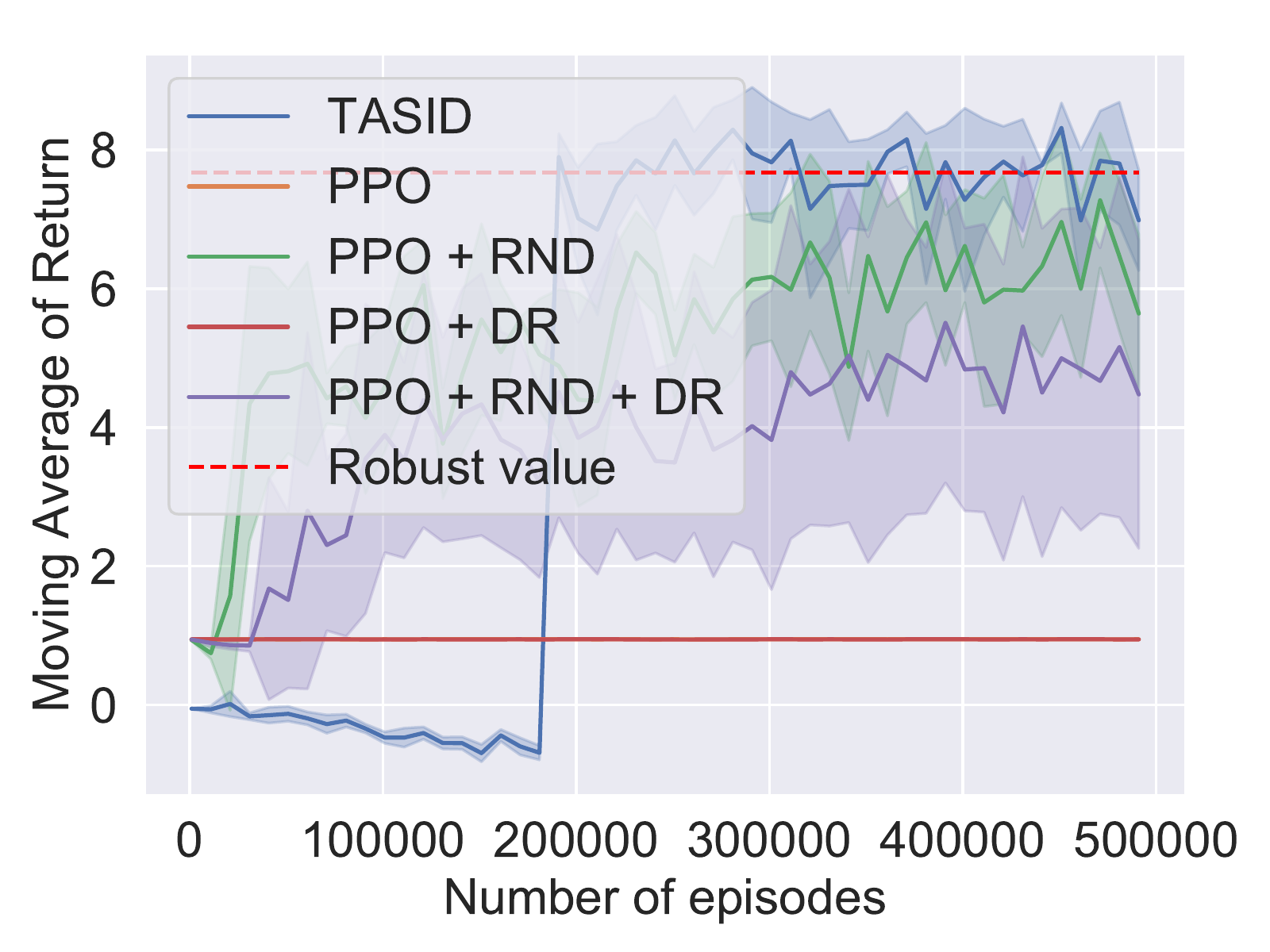}
        \subcaption{$H=20$}
        \label{fig:reward_curve_h20_combolock}
    \end{minipage}
    \caption{Reward curves of different algorithms in the target environments.}
    \label{fig:combo_lock_rewards}
    \figsqueeze
\end{figure*}

\paragraph{Amount of compute} We run our experiments on a cluster containing mixture of P40, P100, and V100 GPUs. Each experiment runs on a single GPU in a docker container. We use Python3 and Pytorch 1.4. We found that for H=40 and V100 GPUs, on average, POTAS took 2.5 hours, PPO + DR took 14 hours, and PPO + RND + DR took 31 hours.  Performing domain randomization increased the computational time by a factor of 2.

\subsection{Experiment Details in MiniGrid}

\paragraph{Domain details} We use the code of the MiniGrid environment \citep{gym_minigrid} under the Apache License 2.0. In MiniGrid, the agent is placed in a discrete grid world and needs to solve different types of tasks. We customized the lava crossing environment in several ways. We first create a shorter but more narrow crossing path, as the optimal path in deterministic environment. Then we construct a longer but more safe path under perturbation. The map of the mini-grid is shown in Figure \ref{fig:minigrid_original}. The height of the map can be change without changing the problem structure and the optimal/robust policy. The state consists of the  $x$-$y$ coordinate of the position, a discrete and 4-values direction, and the time step. The agent only observe a $7\times7$ area in face of it. (See \cite{gym_minigrid} for details.) We set the observation to be the visual map (RGB image) with a random noise between $(50, 50, 50)$ and $(150,150,150)$ for all empty grids. The action space is changed to having five actions: moving forward, turning left, turning right, turning right and moving forward, turning left and moving forward. We truncate the horizon by the three times number of steps the robust policy needs in the deterministic environment. 

\paragraph{Implementation details of $\algname$}
We describe the details of $\algname$ and hyper-parameters below. We implement the action predictor by two convolutional layers and a linear layer, with ReLU activations. We train the algorithm with $5\times10^5$ episodes, for all the $\perturbradius$ values and heights in Figure \ref{fig:minigrid}. in Figure \ref{fig:minigrid}, all the results in the MiniGrid experiments are averaged over 5 runs. Shaded error bars are plus and minus a standard deviation.  Other hyperparameters are shown in Table \ref{tab:hp_ours_minigrid}. 
\begin{table}[ht]
    \centering
    \begin{tabular}{c|c}
        \toprule
         parameter name & values  \\
         \midrule
         CNN kernel 1 & $8\times8\times16$, stride 4 \\
         CNN kernel 2 & $4\times4\times32$, stride 2 \\
         learning rate & $0.0003$ \\
         optimizer & Adam \\
         batch size & $256$ \\
         gradient clipping & 100 \\
         \bottomrule 
    \end{tabular}
    \vspace{0.1cm}
    \caption{Hyperparameters for $\algname$ in the MiniGrid domain}
    \label{tab:hp_ours_minigrid}
\end{table}

\paragraph{Amount of compute} We run our experiments on a cluster containing P100 GPUs. Each experiment runs on a single GPU in a docker container. We use Python3 and Pytorch 1.4. We found that on P100 GPUs, POTAS took 2 to 6 hours (2 hours for $\perturbradius = 0.1$ and 6 hours for $\perturbradius = 0.5$).

\end{document}